\documentclass{article}

\usepackage{verbatim}
\usepackage{amsmath}
\usepackage{amssymb}
\usepackage{natbib}
\newcommand{\tends}{\rightarrow}

\newcommand\argmax{\mathop{\mbox{{\rm argmax}}}\limits}

\newcommand\E{\mathop{\mbox{{\rm E}}}\limits}
\newcommand\diag{{\mbox{diag}}}
\newcommand\tr{{\mbox{tr}}}
\newcommand\rank{{\mbox{rank}}}
\newcommand{\bfSigma}{{\bf \Sigma}}
\newcommand{\bfOmega}{{\bf \Omega}}
\newcommand{\bfA}{{\bf A}}
\newcommand{\bfB}{{\bf B}}
\newcommand{\bfC}{{\bf C}}
\newcommand{\bfD}{{\bf D}}
\newcommand{\bfF}{{\bf F}}
\newcommand{\bfG}{{\bf G}}
\newcommand{\bfH}{{\bf H}}
\newcommand{\bfI}{{\bf I}}
\newcommand{\bfL}{{\bf L}}
\newcommand{\bfQ}{{\bf Q}}
\newcommand{\bfR}{{\bf R}}
\newcommand{\bfS}{{\bf S}}
\newcommand{\bfT}{{\bf T}}
\newcommand{\bfU}{{\bf U}}
\newcommand{\bfV}{{\bf V}}

\newcommand{\bfa}{{\bf a}}
\newcommand{\bfb}{{\bf b}}
\newcommand{\bfe}{{\bf e}}
\newcommand{\bff}{{\bf f}}
\newcommand{\bfu}{{\bf u}}

\newcommand{\bfx}{{\bf x}}
\newcommand{\bfy}{{\bf y}}
\newcommand{\bfz}{{\bf z}}
\newcommand{\calI}{{\cal I}}
\newcommand{\calN}{{\cal N}}
\newcommand{\calX}{{\cal X}}
\newcommand{\calU}{{\cal U}}

\newcommand{\spsdM}{\mathbb{S}_+^M}
\newcommand{\tp}{{\rm T}}
\usepackage{graphicx}
\usepackage{subfig}
\usepackage{algorithm}
\usepackage{algorithmic}

\newtheorem{theorem}{Theorem}
\newtheorem{lemma}{Lemma}

\newtheorem{proposition}{Proposition}
\newtheorem{theorem_dup}{Theorem}
\newtheorem{proposition_dup}{Proposition} 
\newcommand{\BlackBox}{\rule{1.5ex}{1.5ex}}  
\newenvironment{proof}{\par\noindent{\bf Proof\ }}{\hfill\BlackBox\\[2mm]}
\bibpunct{(}{)}{;}{a}{,}{,}

\newenvironment{packed_enum}{
\begin{enumerate}
  \setlength{\itemsep}{1pt}
  \setlength{\parskip}{0pt}
  \setlength{\parsep}{0pt}
}{\end{enumerate}}

\begin{document}

\title{Learning a Factor Model via Regularized PCA}

\author{Yi-Hao Kao \\ Stanford University \\ yhkao@alumni.stanford.edu \and Benjamin Van Roy \\ Stanford University \\ bvr@stanford.edu}

\date{}

\maketitle

\begin{abstract}
We consider the problem of learning a linear factor model. We propose a regularized form of principal component analysis (PCA) and demonstrate through experiments with synthetic and real data the superiority of resulting estimates to those produced by pre-existing factor analysis approaches. We also establish theoretical results that explain how our algorithm corrects the biases induced by conventional approaches.  An important feature of our algorithm is that its computational requirements are similar to those of PCA, which enjoys wide use in large part due to its efficiency.
\end{abstract}

\section{Introduction}

Linear factor models have been widely used for a long time and with notable success in economics, finance, medicine, psychology, and various other 
natural and social sciences \citep{Harman76}.  In such a model, each observed variable is a linear combination of unobserved common factors plus idiosyncratic noise,
and the collection of random variables is jointly Gaussian.
We consider in this paper the problem of learning a factor model from a training set of vector observations.  In particular, our learning problem entails simultaneously estimating the loadings of each factor and the residual variance of each variable.  We seek an estimate of these parameters that best explains out-of-sample data.  For this purpose, we consider the likelihood of test data that is independent of the training data.  As such, our goal is to design a learning algorithm that maximizes the likelihood of a test set that is not used in the learning process.

A common approach to factor model learning involves application of principal component analysis (PCA).  If the number of factors is known
and residual variances are assumed to be uniform, PCA can be applied to efficiently compute model parameters that maximize likelihood of the
training data \citep{Tipping99}.  In order to simplify analysis, we begin our study with a context for which PCA is ideally suited.
In particular, before treating more general models, we will restrict
attention to models in which residual variances are uniform.  As a baseline among learning algorithms, we consider applying PCA together with cross-validation,
computing likelihood-maximizing parameters for different numbers of factors and selecting the number of factors that maximizes likelihood
of a portion of the training data that is reserved for validation.  We will refer to this baseline as  {\it uniform-residual rank-constrained maximum-likelihood} (URM) estimation.

To improve on URM, we propose \emph{uniform-residual trace-penalized maximum-likelihood} (UTM) estimation.  Rather than
estimating parameters of a model with a fixed number of factors and iterating over the number of factors, this approach maximizes likelihood across models without restricting the
number of factors but instead penalizing the trace of a matrix derived from the model's covariance matrix.  
This trace penalty serves to regularize the model and naturally selects a parsimonious set of factors.  The coefficient
of this penalty is chosen via cross-validation, similarly with the way in which the number of factors is selected by URM.
Through a computational study using synthetic data, we demonstrate that UTM results in better estimates than  URM.  
In particular, we find that UTM requires as little as two-thirds of the quantity of data used by URM to match its performance.
Further, leveraging recent work on random matrix theory, we establish theoretical results that explain how UTM corrects the biases induced by URM.

We then extend UTM to address the more general and practically relevant learning problem in which residual variances are not assumed to be uniform.  To evaluate the resulting algorithm, which we refer to as \emph{scaled trace-penalized maximum-likelihood} (STM) estimation,
we carry out experiments using both synthetic data and real stock price data.  The computational results demonstrate that STM leads to more accurate estimates than alternatives available from prior art.  We also provide an analysis to illustrate how these alternatives can suffer from biases in this nonuniform residual variance setting.

Aside from the aforementioned empirical and theoretical analyses, an important contribution of this paper is in the design of algorithms that make UTM and STM efficient. The UTM approach is formulated as a convex semidefinite program (SDP), which can be solved by existing algorithms such as interior-point methods or alternating direction method of multipliers (see, e.g., \citet{Boyd11}). 
However, when the data dimension is large, as is the case in many relevant contexts, such algorithms can take too long to be practically useful.
This exemplifies a recurring obstacle that arises in the use of SDP formulations to study large data sets.  We propose an algorithm based on PCA that solves the UTM formulation efficiently.  In particular, we found this method to typically require three orders of magnitude less compute time than the alternating direction method of multipliers.  Variations of PCA such as URM have enjoyed wide use to a large extent because of their efficiency, and the computation time required for UTM is essentially the same as that of URM.
STM requires additional computation but remains in reach for problems where the computational costs of URM are acceptable.

Our formulation is related to that of \cite{Chandrasekaran10}, which estimates a factor model using a similar trace penalty.  There are some important differences, however, that distinguish our work.  First, the analysis of \cite{Chandrasekaran10} focuses on establishing perfect recovery of structure in an asymptotic regime, whereas our work makes the point that this trace penalty reduces nonasymptotic bias.  Second, our approach to dealing with nonuniform residual variances is distinctive and we demonstrate through computational and theoretical analysis that this difference reduces bias.  Third, \cite{Chandrasekaran10} treats the problem as a semidefinite program, whose solution is often computationally demanding when data dimension is large.  We provide an algorithm based on PCA that efficiently solves our problem.  The algorithm can also be adapted to solve the formulation of \cite{Chandrasekaran10}, though that is not the aim of our paper.

In addition, there is another thread of research on regularized maximum-likelihood estimation for covariance matrices that relates loosely to this paper. Along this line, \cite{Banerjee08} regularizes maximum-likelihood estimation by the $\ell_1$ norm of the inverse covariance matrix in order to recover a sparse graphical model.  An efficient algorithm called graphical Lasso was then proposed by \citet{Friedman08} for solving this formulation. Similar formulations can also be found in \cite{Yuan07} and \cite{Ravikumar11}, who instead penalize the $\ell_1$ norm of off-diagonal elements of the inverse covariance matrix when computing maximum-likelihood estimates. For a detailed survey, see \citet{Pourahmadi10}. Although our approach shares some of the spirit represented by this line of research in that we also regularize maximum-likelihood estimation by an $\ell_1$-like penalty, the settings are fundamentally different: while ours focuses on a factor model, theirs are based on sparse graphical models. We propose an approach that corrects the bias induced by conventional factor analysis, whereas their results are mainly concerned with accurate recovery of the topology of an underlying graph. As such, their work does not address biases in covariance estimates. On the algorithmic front, we develop a simple and efficient solution method that builds on PCA. On the contrary, their algorithms are more complicated and computationally demanding.
\footnote{The code of our algorithms can be downloaded at: \texttt{http://www.yhkao.com/RPCA-code.zip}.}

\section{Problem Formulation}

We consider the problem of learning a factor model without knowledge of the number of factors.  Specifically,
we want to estimate a $M \times M$ covariance matrix $\bfSigma_*$ from samples ${\bf x}_{(1)}, \ldots, {\bf x}_{(N)} \sim {\cal N}(0, \bfSigma_*)$, where $\bfSigma_*$ is the sum of a symmetric matrix $\bfF_*\succeq 0$ and a diagonal matrix $\bfR_* \succeq 0 $.
These samples can be thought of as generated by a factor model of the form ${\bf x}_{(n)} = {\bf F}_*^{\frac{1}{2}} {\bf z}_{(n)} + {\bf w}_{(n)}$,
where ${\bf z}_{(n)} \sim {\cal N}(0, \bfI )$ represents a set of common factors and ${\bf w}_{(n)} \sim {\cal N}(0, \bfR_* )$ represents residual noise.  The number of factors is represented by $\rank(\bfF_*)$, and it is 
usually assumed to be much smaller than the dimension $M$.

Our goal is to produce based on the observed samples a factor loadings 
matrix $\bfF \succeq 0$ and a residual variance matrix $\bfR \succeq 0 $ such that the resulting factor model best explains out-of-sample data. In particular, we seek a pair of $(\bfF,\bfR)$ such that the covariance matrix $\bfSigma=\bfF+\bfR$ maximizes the average log-likelihood of out-of-sample data: $$L(\bfSigma, \bfSigma_*) \triangleq
\E_{\bfx \sim \calN(0,\bfSigma_*) } \left[\log p\left(\bfx | \bfSigma\right) \right] = -\frac{1}{2} \left(M \log(2\pi) + \log \det(\bfSigma) + \mbox{tr}(\bfSigma^{-1} \bfSigma_*) \right).$$
This is also equivalent to minimizing the Kullback-–Leibler divergence between ${\cal N}(0, \bfSigma_*)$ and ${\cal N}(0, \bfSigma)$.

\section{Learning Algorithms}

Given our objective, one simple approach is to choose an estimate $\bfSigma$ that maximizes in-sample log-likelihood:
\begin{equation}
\log p({\cal X} | \bfSigma)  = -\frac{N}{2} \left(M \log(2\pi) + \log\det(\bfSigma) + \mbox{tr}(\bfSigma^{-1} \bfSigma_{\rm SAM}) \right), \label{eq:logpX}
\end{equation}
where ${\cal X} = \{ \bfx_{(1)},\ldots, \bfx_{(N)} \}$, and we use $\bfSigma_{\rm SAM}= \sum_{n=1}^N \bfx_{(n)} \bfx_{(n)}^{\rm T} / N$
to denote the sample covariance matrix.  Here, the maximum likelihood estimate is simply given by $\bfSigma = \bfSigma_{\rm SAM}$.

The problem with maximum likelihood estimation in this context is that in-sample log-likelihood
does not accurately predict out-of-sample log-likelihood unless the number of samples $N$ far exceeds the dimension $M$.  In fact, when the number of samples $N$ is smaller than the dimension $M$, $\bfSigma_{\rm SAM}$ is ill-conditioned and the out-of-sample log-likelihood is negative infinity. One remedy to such poor generalization involves exploiting factor structure, as we discuss in this section.

\subsection{Uniform Residual Variances}

We begin with a simplified scenario in which the residual variances are assumed to be identical. As we will later see, such simplification facilitates theoretical analysis.
This assumption will be relaxed in the next subsection.

\subsubsection{Constraining the Number of Factors}
\label{sec:cons_num_factor}

Given a belief that the data is generated by a factor model with few factors, one natural approach is to employ maximum likelihood estimation with a constraint on the number of factors. 
Now suppose the residual variances in the generative model are identical, and as a result we impose an additional assumption that $\bfR$ is a multiple $\sigma^2 \bfI$ of the identity matrix.
This leads to an optimization problem
\begin{eqnarray}
\max_{\bfF \in \spsdM, \sigma^2 \in \mathbb{R}_+ }  && \log p({\cal X} | \bfSigma) \label{eq:reg_rank_unif}\\ 
\mbox{s.t.} &&  \bfSigma = \bfF + \sigma^2 \bfI \nonumber  \\
&&  {\rm rank}(\bfF) \leq K \nonumber
\end{eqnarray}
where $\spsdM$ denote the set of all $M\times M$ positive semidefinite symmetric matrices, and $K$ is the exogenously specified number of factors. In this case, we can efficiently compute an analytical solution via principal component analysis (PCA), as established in \citet{Tipping99}.  
This involves first computing an eigendecomposition of the sample covariance matrix $\bfSigma_{\rm SAM} = {\bf B S B}^{\rm T}$,
where ${\bf B} = [ \bfb_1 \quad \ldots \quad \bfb_M ]$ is orthonormal and $\bfS=\diag(s_1,\ldots, s_M)$ with $s_1 \geq \ldots \geq s_M$.  The solution to (\ref{eq:reg_rank_unif}) is then given by
\begin{eqnarray}
\hat{\sigma}^2 &=& \frac{1}{M - K} \sum_{i=K+1}^M s_i \nonumber \\
\hat{\bfF} &=& \sum_{k=1}^K (s_k-\hat{\sigma}^2) \bfb_k \bfb_k^\tp. \label{eq:URM_F}
\end{eqnarray}
In other words, the estimate for residual variance equals the average of the last $M-K$ sample eigenvalues, whereas the estimate for factor loading matrix is spanned by the top $K$ sample eigenvectors with coefficients $s_k-\hat{\sigma}^2$.  We will refer to this method as \emph{uniform-residual rank-constrained maximum-likelihood} estimation, and use $\bfSigma^K_{\rm URM} = \hat{\bfF} + \hat{\sigma}^2\bfI$ to denote the covariance matrix resulting from this procedure. It is easy to see that the eigenvalues of $\bfSigma^K_{\rm URM}$ are $s_1, \ldots, s_K, \hat{\sigma}^2, \ldots, \hat{\sigma}^2$, as illustrated in Figure \ref{fig:eigs}(a).

A number of methods have been proposed for estimating the number of factors $K$ \citep{Akaike87, Bishop98, Minka00, Hirose11}.
Cross-validation provides a conceptually simple approach that in practice works at least about as well as any other.  To obtain best performance
from such a procedure, one would make use of so-called $n$-fold cross-validation.  To keep things simple in our study and comparison of estimation
methods, for all methods we will consider, we employ a version of cross-validation that reserves a single subset of data for validation and selection of $K$.  Details of the procedure
we used can be found in the appendix.  Through selection of $K$, this procedure arrives at a covariance matrix which we will denote by $\bfSigma_{\rm URM}$.

\subsubsection{Penalizing the Trace}

Although (\ref{eq:reg_rank_unif}) can be elegantly solved via PCA, it is unclear that imposing a hard constraint on the number of factors will lead to an optimal estimate.  In particular, one might suspect a ``softer" regularization could improve estimation accuracy. Motivated by this idea, we propose penalizing the trace instead of constraining the rank of the factor loading matrix.  As we shall see in the experiment results and theoretical analysis, such an approach indeed improves estimation accuracy significantly.

Nevertheless, naively replacing the rank constraint of (\ref{eq:reg_rank_unif}) by a trace constraint $\mbox{tr}(\bfF) \leq t$ will result in a non-convex optimization problem, and it is not clear to us whether it can be solved efficiently.  Let us explore a related alternative.  Some straightforward matrix algebra shows that if $\bfSigma = \bfF + \sigma^2 \bfI$ with $\bfF \in \spsdM$ and $\sigma^2 > 0$, then the matrix defined by 
$\bfG = \sigma^{-2}\bfI - \bfSigma^{-1}$ is in ${\mathbb S}^M_+ $, with ${\rm rank}(\bfG) = {\rm rank}(\bfF)$.
This observation, together with the
well-known fact that the log-likelihood of ${\cal X}$ is concave in the {\it inverse} covariance matrix $\bfSigma^{-1}$,
motivates the following convex program:
\begin{eqnarray}
\max_{\bfG \in \spsdM, v \in \mathbb{R}_+} &&  \log p({\cal X} | \bfSigma) \nonumber \\
{\rm s.t.} && \bfSigma^{-1} = v \bfI - \bfG \nonumber \\
&& {\rm tr}( {\bfG} ) \leq t. \nonumber
\nonumber
\end{eqnarray}
Here, the variable $v$ represents the reciprocal of residual variance. Pricing out the trace constraint leads to a closely related problem in which the trace is penalized rather than constrained:
\begin{eqnarray}
\max_{\bfG \in \spsdM, v \in \mathbb{R}_+ }  &&  \log p({\cal X} | \bfSigma )  - \lambda {\rm tr}( {\bfG} ) \label{eq:reg_cvx_penalty_simplified} \\
{\rm s.t.} && \bfSigma^{-1} = v\bfI - \bfG. \nonumber
\nonumber
\end{eqnarray}
We will consider the trace penalized problem instead of the trace constrained problem because it is more convenient to design algorithms that 
address the penalty rather than the constraint. Let $(\hat{\bfG},\hat{v})$ be an optimal solution to (\ref{eq:reg_cvx_penalty_simplified}), and let $\bfSigma_{\rm UTM}^\lambda = (\hat{v}\bfI-\hat{\bfG})^{-1}$ denote the covariance matrix estimate derived from it.  Here, the ``U'' indicates that residual variances are assumed to be \emph{uniform} across variables and ``T'' stands for \emph{trace-penalized}.  

It is easy to see that (\ref{eq:reg_cvx_penalty_simplified}) is a semidefinite program.  As such, the problem can be solved in polynomial time by existing algorithms such as interior-point methods or alternating direction method of multipliers (ADMM). However, when the number of variables $M$ is large, as is the case in many contexts of practical import, such algorithms can take too long to be practically useful.
One contribution of this paper is an efficient method for solving (\ref{eq:reg_cvx_penalty_simplified}), which we now describe. 
The following result motivates the algorithm we will propose for computing $\bfSigma_{\rm UTM}^\lambda$:
\begin{theorem}
\label{thm:UTM}
$\bfSigma_{\rm SAM}$ and $\bfSigma_{\rm UTM}^\lambda$ share the same
trace and eigenvectors, and letting the eigenvalues of the two matrices, sorted in decreasing order, be denoted by $s_1,\ldots,s_M$ and $h_1,\ldots,h_M$, respectively, we have
\begin{equation}
\label{eq:soft_thr}
h_m = \max \left\{ s_m - \frac{2\lambda}{N},\frac{1}{\hat{v}} \right\}, \mbox{ for } m=1,\ldots, M. 
\end{equation}
\end{theorem}
This theorem suggests an algorithm for computing $\bfSigma_{\rm UTM}^\lambda$.  First, we compute the eigendecomposition of $\bfSigma_{\rm SAM} = \bfB \bfS \bfB^\tp$, where $\bfB$ and $\bfS$ are as defined in Section \ref{sec:cons_num_factor}.  This
provides the eigenvectors and trace of $\bfSigma_{\rm UTM}^\lambda$.  To obtain its eigenvalues, we only need to determine the value of $\hat{v}$ such that the eigenvalues given by (\ref{eq:soft_thr}) sum to the desired trace.  This is equivalent to determining the largest integer $K$ such that
$$ s_K - \frac{2\lambda}{N} > \frac{1}{M-K} \left( K \cdot \frac{2\lambda}{N}  + \sum_{m=K+1}^M s_m \right). $$
To see this, note that setting
\begin{eqnarray*}
\hat{v}^{-1} & = & \frac{1}{M-K} \left( K \cdot \frac{2\lambda}{N}  + \sum_{m=K+1}^M s_m \right) \\
h_{m} & = &\left\{ \begin{array}{ll}
s_{m}-\frac{2\lambda}{N} & , m = 1, \ldots, K \\
\hat{v}^{-1} & , m=K+1,\ldots, M
\end{array} \right.
\end{eqnarray*}
uniquely satisfies (\ref{eq:soft_thr}) and ensures $\sum_{m=1}^M h_m = \sum_{m=1}^M s_m$.  Algorithm \ref{alg:rURM} presents this method in greater detail. In our experiments, we found this method to typically require three orders of magnitude less compute time than ADMM.  For example, it can solve a problem of dimension $M = 1000$ within seconds on a workstation, whereas ADMM requires hours to attain the same level of accuracy.  

Also note that for reasonably large $\lambda$, this algorithm will flatten most sample eigenvalues and allow only the largest eigenvalues to remain outstanding, effectively producing a factor model estimate.  Figure \ref{fig:eigs}(b) illustrates this effect.  Comparing Figure \ref{fig:eigs}(a) and Figure \ref{fig:eigs}(b), it is easy to see that URM and UTM primarily differ in the largest eigenvalues they produce: while URM simply retains the largest sample eigenvalues, UTM subtracts a constant $2\lambda/N$ from them.  As we shall see in the theoretical analysis, this subtraction indeed corrects the bias incurred in sample eigenvalues.

\begin{figure}[h]
\centering
	\subfloat[]{\includegraphics[scale=0.7]{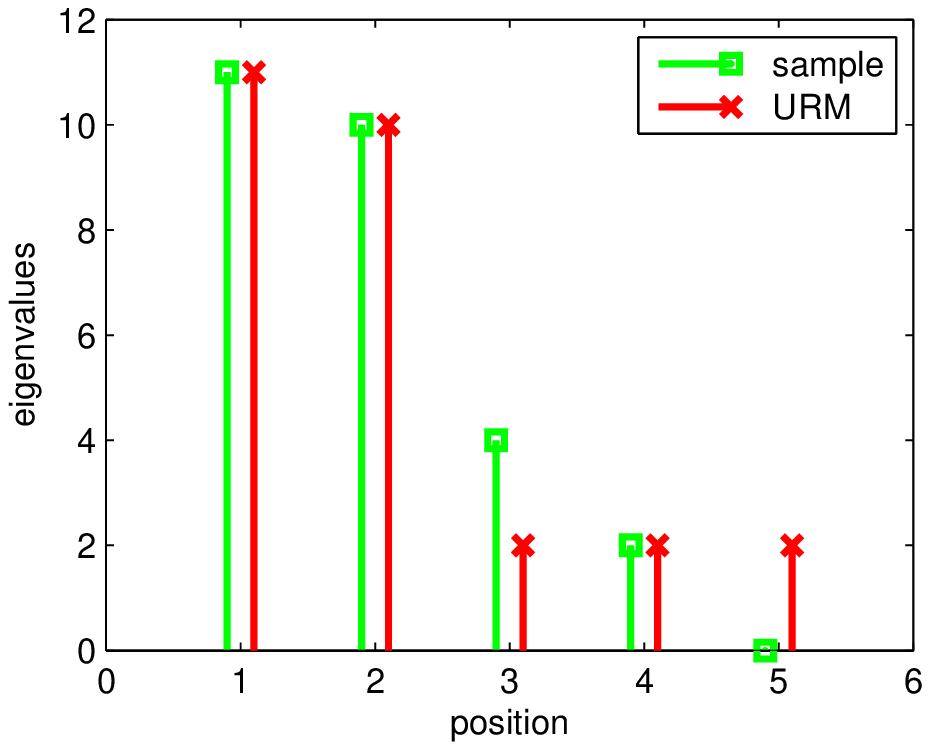}}
	\subfloat[]{\includegraphics[scale=0.7]{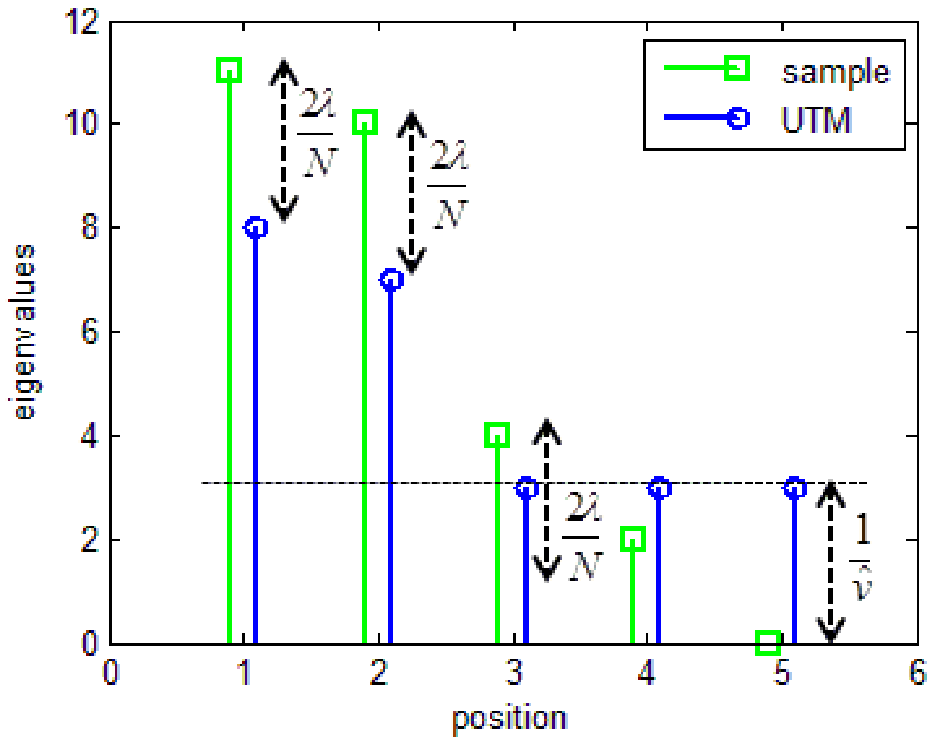}}
\caption{(a) An example of sample eigenvalues and the corresponding eigenvalues of URM estimate, with $M=5$ and $K=2$. URM essentially preserves the top eigenvalues and averages the remaining ones as residual variance. (b) An example of sample eigenvalues and the corresponding eigenvalues of UTM estimate. With a particular choice of $\lambda$, this estimate has two outstanding eigenvalues, but their magnitudes are $2\lambda/N$ below the sample ones. Its residual variance $1/\hat{v}$ is determined in a way that ensures the summation of UTM eigenvalues equal to the sample one.}
\label{fig:eigs}
\end{figure}

Like URM, the most computationally expensive step of UTM lies in the eigendecomposition. 
\footnote{In fact, a full eigendecomposition is not required, as we only need the top $K$ eigenvectors to compute the estimate.} Beyond that, the evaluation of UTM eigenvalues for any given $\lambda$ takes $O(M)$, and is generally negligible.
In our implementation, this regularization parameter $\lambda$ is chosen by cross-validation from a range around $M\hat{\sigma}^2$, whose reasons will become apparent in Section \ref{sec:analysis_uniform}.  We denote by $\bfSigma_{\rm UTM}$ the covariance matrix resulting from this cross-validation procedure.

\begin{algorithm}[h]
\caption{Procedure for computing $\Sigma_{\rm UTM}^\lambda$ \\ {\bf Input: }${\cal X}, \lambda$ \\ {\bfseries Output: }$\bfSigma_{\rm UTM}^\lambda$ }
\label{alg:rURM}
\begin{algorithmic}
\STATE Compute eigendecomposition $\bfSigma_{\rm SAM} = \bfB \bfS \bfB^\tp$ 
\STATE $ v_k^{-1} \leftarrow \frac{1}{M - k}\left( k \cdot \frac{2\lambda}{N} + \sum_{m=k+1}^M s_{m} \right), \quad \forall k=0,1, \ldots, M-1 $
\STATE $K \leftarrow \max \left\{  k: s_{k} - \frac{2 \lambda}{N} > v_k^{-1} \right\}  $ // define $s_0=\infty$
\STATE $ \hat{v} \leftarrow v_K $
\STATE $h_{m} \leftarrow \left\{ \begin{array}{ll}
s_{m}-\frac{2\lambda}{N} & \mbox{if } m \leq K \\
\hat{v}^{-1} & \mbox{otherwise}
\end{array} \right. , \quad \forall m=1,\ldots, M $
\STATE $\bfSigma_{\rm UTM}^\lambda \leftarrow \sum_{k=1}^K (h_k - \hat{v}^{-1} )\bfb_k \bfb_k^{\rm T} + \hat{v}^{-1} \bfI $
\end{algorithmic}
\end{algorithm}

\subsection{Nonuniform Residual Variances}

We now relax the assumption of uniform residual variances and discuss several methods for the general case.
As in the previous subsection, the hyper-parameters of these methods will be selected by cross-validation.

\subsubsection{Constraining the Number of Factors}

Without the assumption of uniform residual variances, the ranked-constrained maximum-likelihood formulation can be written as
\begin{eqnarray}
\max_{\bfF \in \spsdM, \bfR \in \mathbb{D}^M_+ }  && \log p({\cal X} | \bfSigma) \label{eq:reg_rank}\\ 
\mbox{s.t.} &&  \bfSigma = \bfF + \bfR \nonumber  \\
&&  {\rm rank}(\bfF) \leq K \nonumber
\end{eqnarray}
where $\mathbb{D}^M_{+}$ denote the set of all $M\times M$ positive semidefinite diagonal matrices. Unlike (\ref{eq:reg_rank_unif}), this formulation is generally hard to solve, and therefore we consider two widely-used approximate solutions.

The expectation-maximization (EM) algorithm \citep{Rubin82} is arguably the most conventional approach to solving (\ref{eq:reg_rank}), though there is no guarantee that this will result in a global optimal solution.
The algorithm generates a sequence of iterates $\bfF^{\frac{1}{2}} \in \mathbb{R}^{M\times K}$ and $\bfR \in \mathbb{D}^M_+$, 
such that the covariance matrix $\bfSigma = \bfF^{\frac{1}{2}} \bfF^{\frac{\rm T}{2}} + \bfR$ increases the log-likelihood
of ${\cal X}$ with each iteration.  Each iteration involves an estimation step in which we assume the data are generated according to the covariance matrix $\bfSigma = \bfF^{\frac{1}{2}} \bfF^{\frac{\rm T}{2}} + \bfR$, and compute expectations $\E[\bfz_{(n)} | \bfx_{(n)}]$ and $\E[\bfz_{(n)} \bfz_{(n)}^{\rm T} | \bfx_{(n)}]$ for $n = 1,\ldots, N$.  A maximization step then updates $\bfF$ and $\bfR$ based on these expectations. In our implementation, the initial $\bfF$ and $\bfR$ are selected by the MRH algorithm described in the next paragraph. We will denote the estimate produced by the EM algorithm by $\bfSigma^K_{\rm EM}$ and that resulting from further selection of $K$ through cross-validation by $\bfSigma_{\rm EM}$.

A common heuristic for approximately solving (\ref{eq:reg_rank}) without entailing iterative computation is to first compute $\bfSigma^K_{\rm URM}$ by PCA and then take the factor matrix estimate to be the $\hat{\bfF}$ defined in (\ref{eq:URM_F}) and the residual variances to be
$\hat{\bfR}_{m,m} = \left( \bfSigma_{\rm SAM} - \hat{\bfF} \right)_{m,m}$, 
for $m = 1,\ldots,M$.  In other words, $\hat{\bfR}_{m,m}$ is selected so that the diagonal elements of the estimated covariance matrix $\hat{\bfSigma} = \hat{\bfF} + \hat{\bfR}$ are equal to those of the sample covariance matrix.  We will refer to this method as \emph{marginal-variance-preserving rank-constrained heuristic} and denote the resulting estimates by 
$\bfSigma_{\rm MRH}^K$ and $\bfSigma_{\rm MRH}$.

\subsubsection{Penalizing the Trace}
\label{sec:DNRV}

We now develop an extension of Algorithm \ref{alg:rURM} that applies when residual variances are nonuniform. 
One formulation that may seem natural involves replacing $v\bfI$ in (\ref{eq:reg_cvx_penalty_simplified}) with a diagonal matrix
$\bfV \in  \mathbb{D}^M_+$. That is,
\begin{eqnarray}
\max_{\bfG \in \spsdM, \bfV \in \mathbb{D}^M_+ }  &&  \log p({\cal X} | \bfSigma )  - \lambda {\rm tr}( {\bfG} ) \label{eq:TM} \\
{\rm s.t.} && \bfSigma^{-1} = \bfV - \bfG. \nonumber
\end{eqnarray}
Indeed, a closely related formulation is proposed in \citet{Chandrasekaran10}.
However, as we will see in Sections \ref{se:experiments} and \ref{sec:analysis}, solutions to this formulation suffer from bias and do not compete well against the method we will propose next.  That said, let us denote the estimates resulting from solving this formulation by 
$\bfSigma_{\rm TM}^\lambda$ and $\bfSigma_{\rm TM}$, where ``T'' stands for \emph{trace-penalized}. 
Also note that this formulation can be efficiently solved by a straightforward generalization of Theorem \ref{thm:UTM}, though we will not elaborate on this.

Our approach involves componentwise scaling of the data. 
Consider an estimate $\hat{\bfSigma}$ of $\bfSigma_*$. Recall that we evaluate the quality of the estimate using the expected log-likelihood $L(\hat{\bfSigma}, \bfSigma_*)$
of out-of-sample data.
If we multiply each data sample by a matrix $\bfT \in \mathbb{R}^{M\times M}$, the data set becomes $\bfT \calX \triangleq \{ \bfT\bfx_{(1)},  \ldots, \bfT\bfx_{(N)} \}$, where $\bfT \bfx_{(n)} \sim {\cal N}(0, \bfT \bfSigma_* \bfT^{\rm T} )$.  If we also change our estimate accordingly to $\bfT \hat{\bfSigma} \bfT^{\rm T}$ then the new expected log-likelihood becomes
$$ L( \bfT \hat{\bfSigma} \bfT^{\rm T}, \bfT \bfSigma_* \bfT^{\rm T}) = L(\hat{\bfSigma}, \bfSigma_*) - \log \det \bfT. $$
Therefore, as long as we constrain $\bfT$ to have unit determinant, $L( \bfT \hat{\bfSigma} \bfT^{\rm T}, \bfT \bfSigma_* \bfT^{\rm T})$ will be equal to $L(\hat{\bfSigma}, \bfSigma_*)$, suggesting that if $\hat{\bfSigma}$ is a good estimate of $\bfSigma_*$ then $ \bfT \hat{\bfSigma} \bfT^{\rm T}$ is a good estimate of $ \bfT \bfSigma_* \bfT^{\rm T}$.
This motivates the following optimization problem:
\begin{eqnarray}
\max_{\bfG \in \spsdM, v \in \mathbb{R}_+, \bfT \in \mathbb{D}_+^M }  &&  \log p( \bfT {\cal X} | \bfSigma ) - \lambda {\rm tr}( {\bfG} ) \label{eq:reg_STM} \\
{\rm s.t.} && \bfSigma^{-1} = v \bfI - \bfG \nonumber. \\
&& \log \det{\bfT} \geq 0 \nonumber.
\nonumber
\end{eqnarray}
The solution to this problem identifies a componentwise-scaling matrix $\bfT\in \mathbb{D}^M_+$ that allows the data to be best-explained by a factor model with uniform residual variances.
Given an optimal solution, $1/\bfT_{i,i}^2$ should be approximately proportional to the residual variance of the $i$th variable, so that scaling by $\bfT_{i,i}$ makes residual
variances uniform.  Note that the optimization problem constrains $\log \det \bfT$ to be nonnegative rather than zero.  This makes the feasible region convex,
and this constraint is binding at the optimal solution.  
Denote the optimal solution to (\ref{eq:reg_STM}) by $(\hat{\bfG}, \hat{v}, \hat{\bfT})$. Our estimate is thus given by $ \hat{\bfT}^{-1}(\hat{v}\bfI-\hat{\bfG})^{-1} \hat{\bfT}^{-\rm T}$.

The objective function of (\ref{eq:reg_STM}) is not concave in $(\bfG, v, \bfT)$, but is biconcave in $(\bfG, v)$ and $\bfT$.  We solve it by coordinate 
ascent, alternating between optimizing $(\bfG,v)$ and $\bfT$.  This procedure is guaranteed convergence.
In our implementation, we initialize $\bfT$ by $\bfI$. We will denote the resulting estimates by $\bfSigma_{\rm STM}$, where ``ST'' stands for \emph{scaled} and \emph{trace-penalized}.

\section{Experiments}
\label{se:experiments}

We carried out two sets of experiments to compare the performance of aforementioned algorithms. The first is based on synthetic data, whereas the second uses historical prices of stocks that make up the S\&P 500 index.

\subsection{Synthetic Data}
\label{se:synthetic}

We generated two kinds of synthetic data.  The first was generated by a model in which each residual has unit variance.  This data was sampled according to the
following procedure, which takes as input the number of factors $K_*$, the dimension $M$, the factor variances $\sigma_f^2$, and the number of samples $N$:
\begin{packed_enum}
	\item Sample $K_*$ orthonormal vectors ${\bf\phi}_1, {\bf\phi}_2, \ldots, {\bf\phi}_{K_*} \in \mathbb R^{M}$ isotropically.
	\item Sample $f_1, f_2,\ldots, f_{K_*} \sim {\cal N}(0, \sigma_f^2)$.
	\item Let ${\bf F}_*^{\frac{1}{2}} = [ f_1 \phi_1 \quad f_2 \phi_2 \quad \ldots \quad f_{K_*} \phi_{K_*} ]$.
	\item Let $\bfSigma_* = {\bfF}_*^{\frac{1}{2}} {\bfF}_*^{\frac{\rm T}{2}} + \bfI$.
	\item Sample ${\bf x}_{(1)}, \ldots, {\bf x}_{(N)}$ iid from $ {\cal N}(0, \bfSigma_*) $.
\end{packed_enum}
We repeated this procedure one hundred times for each $N \in \{50, 100, 200, 400 \}$, with $M=200, K_*=10$, and $\sigma_f = 5$.  We applied to this data
URM and UTM, since they are methods designed to treat such a scenario with uniform residual variances.  Regularization parameters 
$K$ and $\lambda$ were selected via cross-validation, where about $70\%$ of each data set was used for training and $30\%$ for validation.
Figure $\ref{fig:syn_unif}$(a) plots out-of-sample log-likelihood delivered by the two algorithms.  Performance is plotted as a function of the log-ratio of the 
number of samples to the number of variables, which represents the availability of data relative to the number of variables.  We expect this measure 
to drive performance differences.  UTM outperforms URM in all scenarios.  
The difference is largest when data is scarce.  When data 
is abundant, both methods work about as well.  This should be expected since both estimation methods are consistent.

To interpret this result in a more tangible way, we also plot the \emph{equivalent data requirement} of UTM in Figure $\ref{fig:syn_unif}$(b). This metric is defined as the portion of training data required by UTM to match the performance of URM. As we can see, UTM needs as little as 67\% of the data used by the URM to reach the same estimation accuracy.

\begin{figure}[h]
\centering
	\subfloat[]{\includegraphics[scale=0.7]{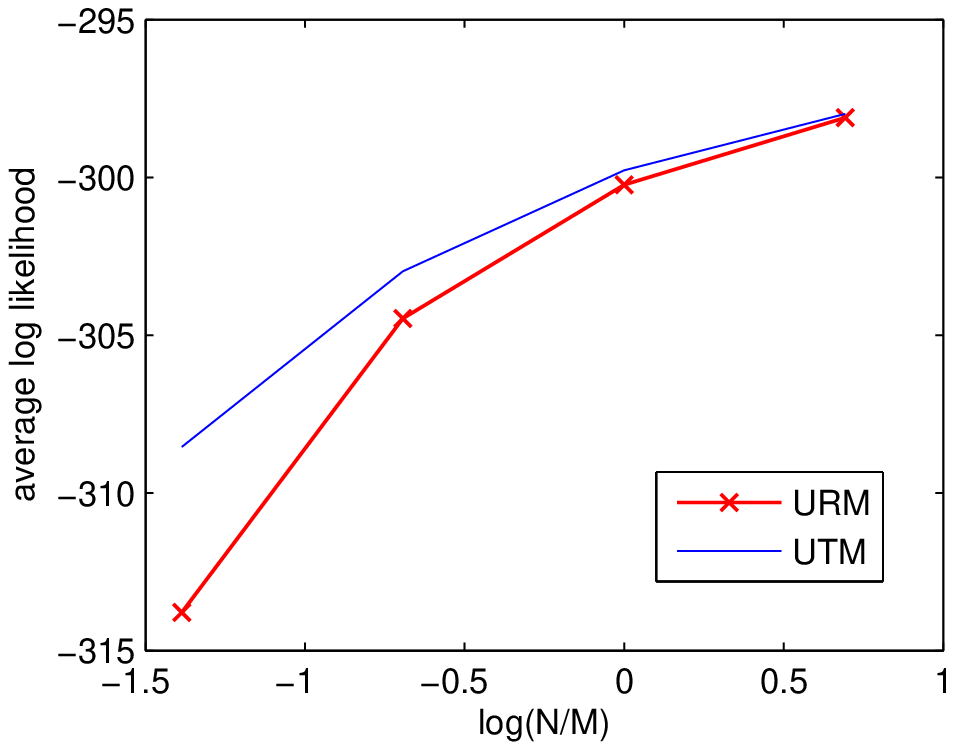}}
	\subfloat[]{\includegraphics[scale=0.7]{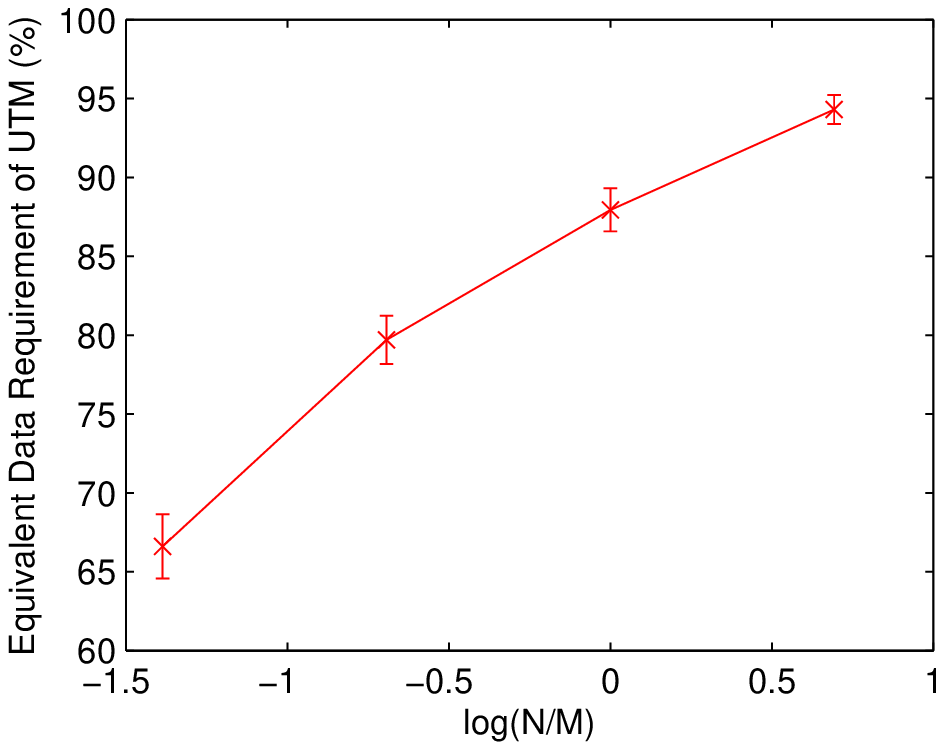}}
\caption{(a) The average out-of-sample log-likelihood delivered by URM and UTM, when residual variances are identical. (b) The average portion of data required by UTM to match the performance of URM. The error bars denote 95\% confidence interval.}
\label{fig:syn_unif}
\end{figure}

Our second type of synthetic data was generated using an entirely similar procedure except step 4 was replaced by
$$ \bfSigma_* = \bfF_*^{\frac{1}{2}} \bfF_*^{\frac{\rm T}{2}} + \mbox{diag}( e^{r_1}, e^{r_2}, \ldots, e^{r_M} ), $$
where $r_1, \ldots, r_M$ were sampled iid from ${\cal N}(0, \sigma_r^2)$. Note that $\sigma_r$ effectively controls the variation among residual variances. Since these residual variances are nonuniform, EM, MRH, TM, and STM were applied. Figure \ref{fig:syn_nonunif_llh} plots the results for the cases $\sigma_r = 0.5$ and $\sigma_r=0.8$, corresponding to moderate and large variation among residual variances, respectively. In either case, STM outperforms the alternatives. Figure $\ref{fig:syn_nonunif_edr}$ further gives the equivalent data requirement of STM with respect to each alternative.  It is worth pointing out that the performance of MRH and TM degrades significantly as the variation among residual variances grows, while EM is less susceptible to such change.  We will elaborate on this phenomenon in Section \ref{sec:analysis_nonunif}.

\begin{figure}[h]
\centering
  \subfloat[]{\includegraphics[scale=0.7]{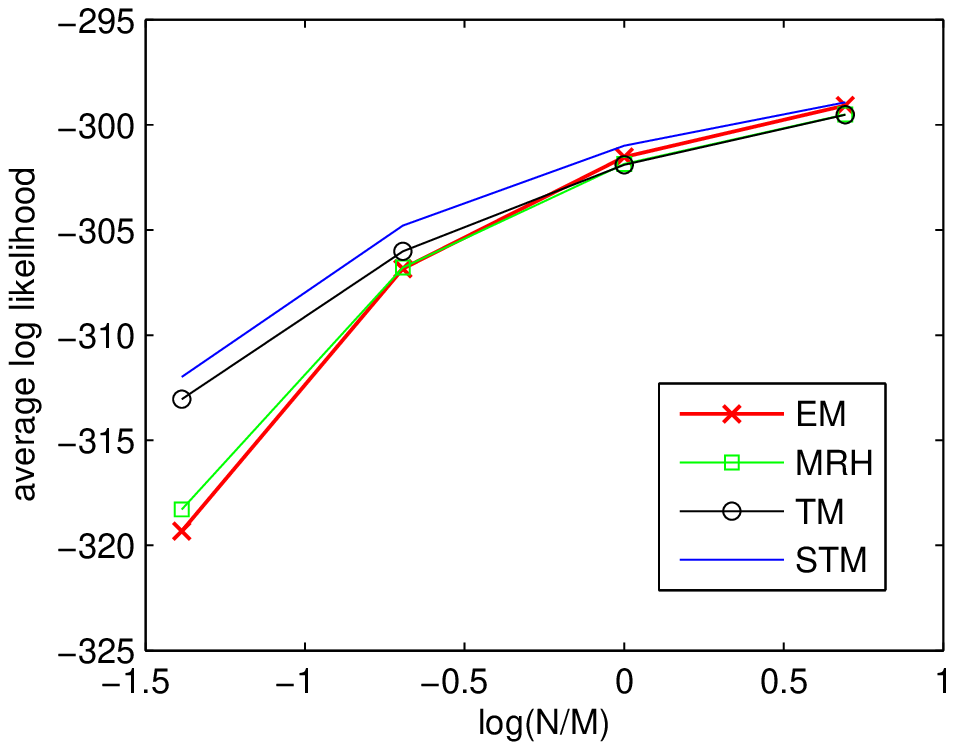}}
  \subfloat[]{\includegraphics[scale=0.7]{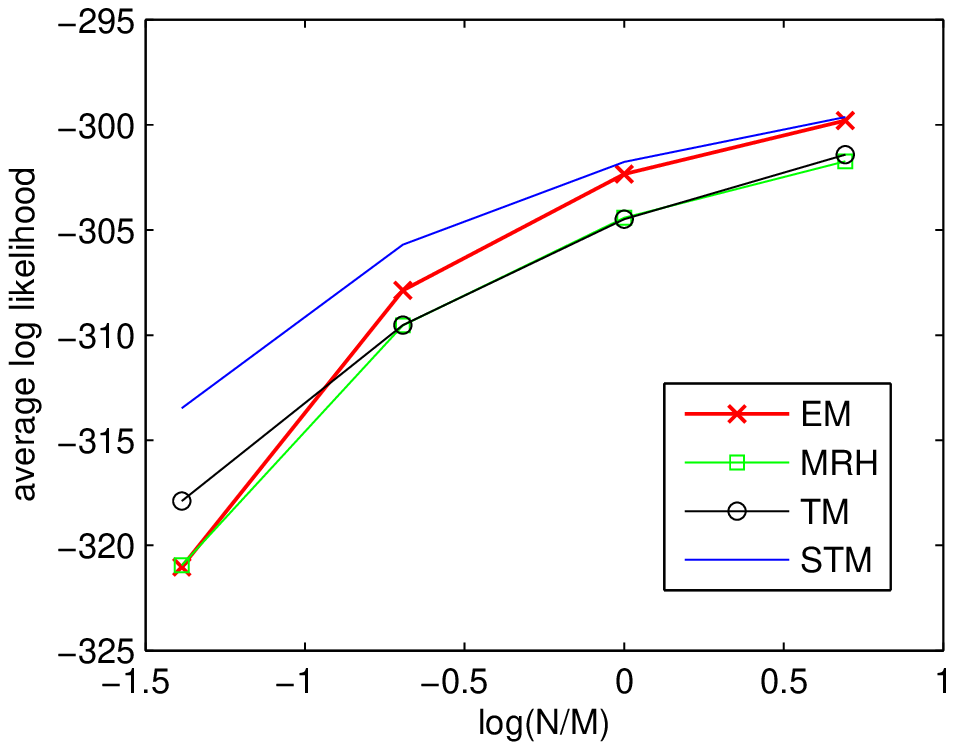}}
  \caption{The average out-of-sample log-likelihood delivered by EM, MRH, TM, and STM, when residuals have independent random variances  with (a) $\sigma_r=0.5$ and (b) $\sigma_r=0.8$.}
\label{fig:syn_nonunif_llh}
\end{figure}

\begin{figure}[h]
\centering
  \subfloat[]{\includegraphics[scale=0.7]{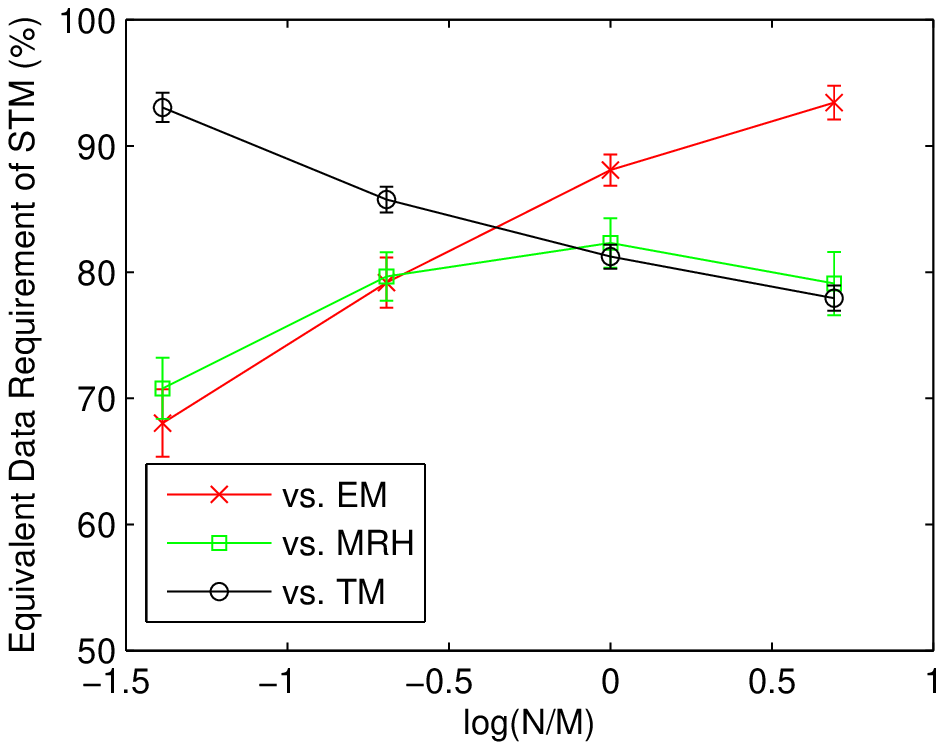}}
  \subfloat[]{\includegraphics[scale=0.7]{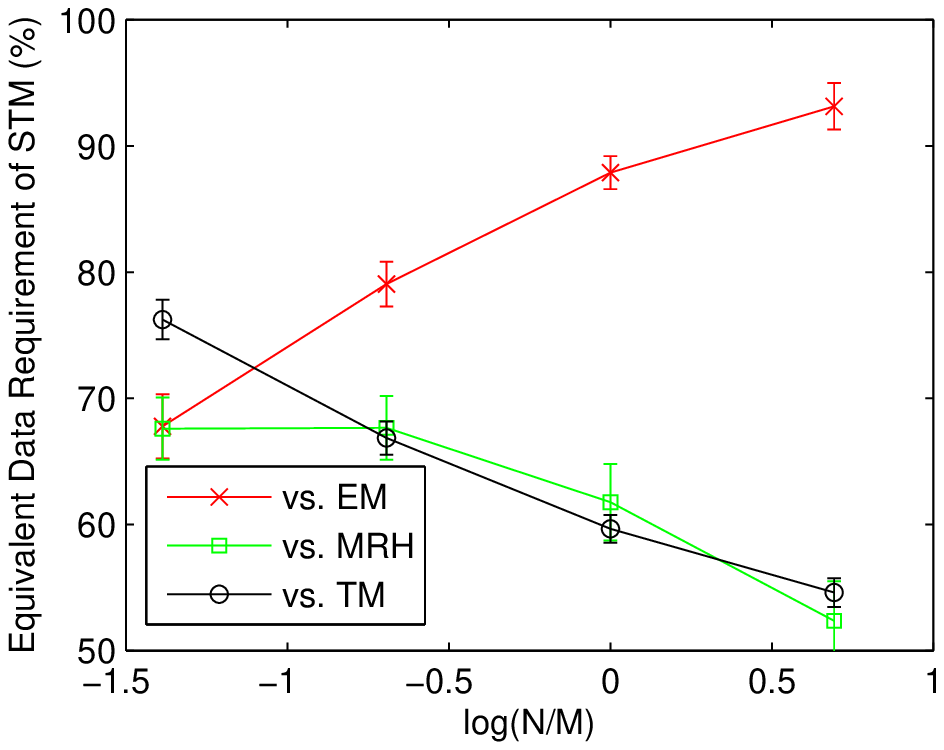}}
  \caption{The equivalent data requirement of STM with respect to EM, MRH, and TM when (a) $\sigma_r=0.5$ and (b) $\sigma_r=0.8$. The error bars denote 95\% confidence interval. }
\label{fig:syn_nonunif_edr}
\end{figure}

\subsection{S\&P 500 Data}

An important application area of factor analysis is finance, where return covariances are used to assess risks and guide diversification \citep{Markowitz52}.
The experiments we will now describe involve estimation of such covariances from historical daily returns of stocks represented in the S\&P 500 index as of
March, 2011.  We use price data collected from the period starting November 2, 2001, and ending August 9, 2007.  This period was chosen to avoid the erratic
market behavior observed during the bursting of the dot-com bubble in 2000 and the financial crisis that began in 2008.  Normalized daily log-returns were 
computed from closing prices through a process described in detail in the appendix.  Over this duration, there were 1400 trading days and 453 of the stocks
under consideration were active.  This produced a data set ${\cal Y} = \{ \bfy_{(1)}, \ldots, \bfy_{(1400)} \}$, 
in which the $i$th component of $\bfy_{(t)}\in \mathbb{R}^{453}$ represents the normalized log-daily-return of stock $i$ on day $t$.

We generated estimates corresponding to each among a subset of the 1400 days.  As would be done in real-time application, for each such day
$t$ we used $N$ data points $\{{\bf y}_{(t-N+1)},$ $\ldots,$ ${\bf y}_{(t)}\}$ that would have been available on that day to compute the estimate and subsequent data 
to assess performance.  In particular, we generated estimates every ten days
beginning on day $1200$ and ending on day $1290$.  For each of these days, we evaluated average log-likelihood of log-daily-returns over the next ten days. Algorithm $\ref{alg:testing}$ formalizes this procedure.
\begin{algorithm}
\caption{Testing Procedure ${\cal T}$\\ {\bf Input:} learning algorithm ${\cal U}$, regularization parameter $\theta$, window size $N$, time point $t$ \\ {\bfseries Output:} test-set log-likelihood }
\label{alg:testing}
\begin{algorithmic}
\STATE ${\cal X} \leftarrow  \{ {\bf y}_{(t-N+1)}, \ldots, {\bf y}_{(t)} \}$ (training set)
\STATE ${\cal X}' \leftarrow  \{ {\bf y}_{(t+1)}, \ldots, {\bf y}_{(t+10)} \}$ (test set)
\STATE $\hat{\Sigma} \leftarrow {\cal U}({\cal X}, \theta)$
\RETURN{ $\log p( {\cal X}' | \hat{\Sigma} )$ }
\end{algorithmic}
\end{algorithm}

These tests served the purpose of sliding-window cross-validation, as we tried a range of regularization parameters over this time period
and used test results to select a regularization parameter for each algorithm. More specifically, for each algorithm $\cal U$, its regularization parameter was selected by
$$ \hat{\theta}= \argmax_{\theta} \sum_{j=0}^{9} {\cal T}({\cal U}, \theta, N, 1200+10 j). $$
On days $1300, 1310, \ldots, 1390$, we generated one estimate per
day per algorithm, in each case using the regularization parameter selected earlier and evaluating average log-likelihood over the next ten days.
For each algorithm $\cal U$, we took the average of these ten ten-day averages to be its out-of-sample performance, defined as
$$ \frac{1}{100} \sum_{j=0}^{9} {\cal T}({\cal U}, \hat{\theta}, N, 1300+10 j). $$

Figure \ref{fig:real} plots the performance delivered EM, MRH, TM, and STM with $N\in \{ 200, 300, \ldots, 1200\}$.  STM is the dominant
solution.  It is natural to ask why the performance of each algorithm improves then degrades as 
$N$ grows.  If the time series were stationary, one would expect performance to monotonically improve with $N$.  However, this is a real time series 
and is not necessarily stationary.  We believe that the distribution changes enough over about a thousand trading days so that using historical
data collected further back worsens estimates.  This observation points out that in real applications STM is likely to generate superior
results even when all aforementioned algorithms are allowed to use all available data.  This is in contrast with the experiments of Section \ref{se:synthetic}
involving synthetic data, which may have led to an impression that the performance difference could be made small by using more data.

\begin{figure}[h]
\centering
\includegraphics[scale=0.75]{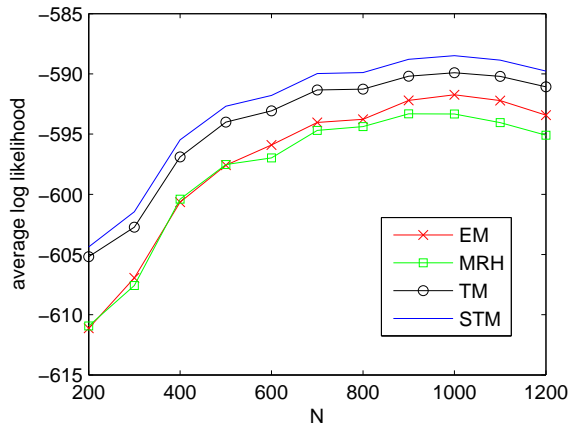}
\caption{The average log-likelihood of test set, delivered by EM, MRH, TM, and STM, over different training-window sizes $N$.}
\label{fig:real}
\end{figure}

\section{Analysis}
\label{sec:analysis}

In this section, we explain why UTM and STM are expected to outperform alternatives as we have seen in our experimental results.

\subsection{Uniform Residual Variances}
\label{sec:analysis_uniform}

Let us start with the simpler context in which residual variances are identical.  In other words, let $\bf\Sigma_* = \bfF_* + \sigma^2 \bfI$ for a low
rank matrix $ \bfF_*$ and uniform residual variance $\sigma^2$.  We will begin our analysis with two desirable properties of UTM, and then move on to the comparison between UTM and URM.

It is easy to see that $\bfSigma_{\rm UTM}^\lambda$ is a consistent estimator of $\bfSigma_*$ for any $\lambda > 0$, since $\lim_{N \tends \infty} \frac{2\lambda}{N} = 0 $ and by Theorem \ref{thm:UTM} we have 
$$ \lim_{ N \tends \infty} \bfSigma_{\rm UTM}^\lambda = \lim_{N \tends \infty} \bfSigma_{\rm SAM} \overset{a.s.}{\longrightarrow} \bfSigma_* .$$
Another important property of UTM is the fact that the trace of UTM estimate is the same as that of sample covariance matrix. This preservation is desirable as suggested by the following result.
\begin{proposition}
\label{prop:trace-preservation}
For any fixed $N, K$, scalars $\ell_1 \geq \ell_2 \geq \cdots \geq \ell_K \geq \sigma^2 > 0$, and any sequence of covariance matrices $\bfSigma_*^{(M)} \in \spsdM$ with eigenvalues $\ell_1,\ldots,\ell_K,\sigma^2,\ldots, \sigma^2$, we have
$$\frac{ {\rm tr} \bfSigma_{\rm SAM} }{ {\rm tr} \bfSigma_* } \overset{a.s.}{\longrightarrow} 1,$$
as $M \tends \infty$ and 
\begin{equation}
\label{eq:trace_bound}
{\rm Pr} \left( \left| \frac{ {\rm tr} \bfSigma_{\rm SAM} }{ {\rm tr} \bfSigma_* } - 1 \right| \geq \epsilon \right) \leq 2\exp\left( - N \epsilon^2 \Omega(M) \right).
 \end{equation}
\end{proposition}
Note that, for any fixed fixed number of samples $N$, the right-hand-side of (\ref{eq:trace_bound}) diminishes towards 0 as data dimension $M$ grows.  In other words, as long as the data dimension is large compared to the number of factors $K$, the sample trace is usually a good estimate of the true one, even when we have very limited data samples.

Now we would like to understand why UTM outperforms URM.  Recall that,
given an eigendecomposition $\bfSigma_{\rm SAM} = \bfB \bfS \bfB^{\rm T}$ of the sample covariance matrix, estimates generated
by URM and UTM admit eigendecompositions $\bfSigma_{\rm URM} = \bfB \bfH_{\rm URM} \bfB^{\rm T}$ and $\bfSigma_{\rm UTM} = \bfB \bfH_{\rm UTM} \bfB^{\rm T}$,
deviating from the sample eigenvalues $\bfS$ but not the eigenvectors $\bfB$.  Hence, URM and UTM differ only in the way they select
eigenvalues: URM takes each eigenvalue to be either a constant or the corresponding 
sample eigenvalue, while  UTM takes each eigenvalue to be either a constant or the corresponding sample eigenvalue less another constant. Thus, large eigenvalues produced by UTM are a constant offset less than those produced by URM, as illustrated in Figure \ref{fig:eigs}.  We now explain why such subtraction lends UTM an advantage over URM in high-dimensional cases.

Given the eigenvectors $\bfB = [\bfb_1 \cdots \bfb_M]$, let us consider the optimal eigenvalues that maximize out-of-sample log-likelihood of the estimate. Specifically, let us define
$$ \bfH^* \triangleq \argmax_{\bfH \in \bfD_+^M} L(\bfB \bfH \bfB^{\rm T}, \bfSigma_*) .$$
With some straightforward algebra, we can show that $\bfH^* = {\rm diag}(h_1^*,\ldots, h_M^*)$, where $h_i^* = \bfb_i^{\rm T} \bfSigma_* \bfb_i,$ for $i=1,\ldots, M$.
Let each $i$th sample eigenvalue be denoted by $s_i = \bfS_{i,i}$, and let the $i$th largest eigenvalue of $\bf\Sigma_*$ be denoted by $\ell_i$.
The following theorem, whose proof relies on two results from random matrix theory found in \citet{Baik06} and \citet{Paul07}, relates sample eigenvalues $s_i$ to optimal eigenvalues $h_i^*$.
\begin{theorem}
\label{th:constant-correction}
For all $K$, scalars $\ell_1 > \ell_2 > \cdots > \ell_K > \sigma^2 > 0$, $\rho \in (0,1)$, sequences $N_{(M)}$ such that $|M/N_{(M)} - \rho| = o(1/\sqrt{N_{(M)}})$, covariance matrices $\bfSigma_*^{(M)} \in \spsdM$ with eigenvalues $\ell_1,\ldots,\ell_K,\sigma^2,\ldots, \sigma^2$, and $i$ such that
$\ell_i > (1+\sqrt{\rho}) \sigma^2$, there exists $\epsilon_i \in (0,2\sigma^2/(\ell_i-\sigma^2))$ such that  
$h_i^* \overset{p}{\longrightarrow} s_i - (2+\epsilon_i) \rho \sigma^2$ as $M \tends \infty$.
\end{theorem}
Consider eigenvalues $\ell_i$ that are large relative to $\sigma^2$ so that $\epsilon_i$ is negligible.  In such cases, when in
the asymptotic regime identified by Theorem \ref{th:constant-correction}, we have $h_i^* \approx s_i - 2 \rho \sigma^2$.  
This observation suggests that, when the number of factors $K$ is relatively small compared to data dimension $M$, and when $M$ and number of samples $N$ scale proportionally to large numbers, the way in which UTM subtracts a constant from large sample eigenvalues should improve performance relative to URM, which does not modify large sample eigenvalues.
Furthermore, comparing Theorem \ref{thm:UTM} and \ref{th:constant-correction}, we can see that the correction term should satisfy $\frac{2\lambda}{N} \simeq 2 \rho \sigma^2$, or equivalently $\lambda \simeq M\sigma^2$.  This relation can help us narrow the search range of $\lambda$ in cross-validation.

It is worth pointing out that the over-shooting effect of sample eigenvalues is well known in statistics literature (see, e.g, \citet{Johnstone01} ).  Our contribution, however, is to quantify this effect for factor models, and show that the large eigenvalues are not only biased high, but biased high by the same amount.

\subsection{Nonuniform Residual Variances}
\label{sec:analysis_nonunif}

Comparing Figure $\ref{fig:syn_unif}$, $\ref{fig:syn_nonunif_llh}$, and $\ref{fig:syn_nonunif_edr}$, we can see that the relation between URM and UTM is analogous to that between  EM and STM. Specifically, the equivalent data requirement of UTM versus URM behaves very similarly as that of STM versus EM. This should not be surprising, as we now explain.

To develop an intuitive understanding of this phenomenon, let us consider an idealized, analytically tractable context in which both EM and STM successfully estimate the relative magnitudes of residual variances. In particular, suppose we impose an additional constraint $\bfR \propto \bfR_*$ into ($\ref{eq:reg_rank}$) and an additional constraint $\bfT \propto \bfR_*^{-\frac{1}{2}}$ into ($\ref{eq:reg_STM}$). \footnote{Here we use the notation $\bfA \propto \bfB$ to mean that there exists $\gamma \geq 0$ such that $\bfA = \gamma \bfB$.} In this case, it is straightforward to show that EM is equivalent to URM with data scaled by $\bfR_*^{-\frac{1}{2}}$, and STM is equivalent to UTM with the same scaled data.
Therefore, by the argument given in Section \ref{sec:analysis_uniform}, it is natural to expect STM outperforms EM.

A question that remains, however, is why MRH and TM are not as effective as STM.  We believe the reason to be that they tend to select factor loadings that assign larger values than appropriate to variables with large residual variances.  Indeed, such disadvantage has been observed in our synthetic data experiment: when the variation among residual variances increases, the performances of MRH and TM degrade significantly, as shown in Figure \ref{fig:syn_nonunif_llh}.

Again, let us illustrate this phenomenon through an idealized context.  Specifically, consider a case in which the sample covariance matrix $\bfSigma_{\rm SAM}$ turns out to be identical to $\bfSigma_* = \bfF_* + \bfR_*$, with $\bfR_* = \diag( r,1,1,\ldots, 1)$ and $\bfF_* = {\bf 1}{\bf 1}^\tp$, where $\bf 1$ is a vector with every component equal to $1$.  Recall that MRH uses the eigenvectors of $\bfSigma_{\rm SAM}$ corresponding to the largest eigenvalues as factor loading vectors.  One would hope that factor loading estimates are insensitive to underlying residual variances.  However, the following proposition suggests that, as $r$ grows, the first component of the first eigenvector of $\bfSigma_{\rm SAM}$ dominates other components by an unbounded ratio.
\begin{proposition} \label{prop:MRH-fail}
Suppose $\bfR_* = \diag( r,1,1,\ldots, 1)$ and $\bfF_* = {\bf 1}{\bf 1}^\tp$, $r > 1$. Let $\bff = [f_1 \quad \ldots \quad f_M]^{\rm T}$ be the top eigenvector of $\bfSigma_*$. Then we have $ f_1 / f_i = \Omega(r), \forall i>1$.
\end{proposition}
As such, the factor estimated by this top eigenvector can be grossly misrepresented, implying MRH is not preferable when residual variances differ significantly from each other.

TM suffers from a similar problem, though possibly to a lesser degree.  The matrix $\bfV$ in the TM formulation (\ref{eq:TM}) represents an estimate of $\bfR_*^{-1}$.
For simplicity, let us consider an idealized TM formulation which further incorporates a constraint $\bfV = \bfR_*^{-1}$. That is,
\begin{eqnarray}
\max_{ \bfV \in \mathbb{D}_+^M, \bfG \in \spsdM}  &&  \log p({\cal X} | \bfSigma )  - \lambda {\rm tr}( {\bfG} ) \label{eq:modified_TM} \\
{\rm s.t.} && \bfSigma^{-1} =  \bfV - \bfG \nonumber \\
&& \bfSigma_{\rm SAM} = \bfSigma_*  \nonumber \\
&& \bfV = \bfR_*^{-1}. \nonumber
\nonumber
\end{eqnarray}
	Using the same setting as in Proposition \ref{prop:MRH-fail}, we can show that when this idealized TM algorithm produces an estimate of exactly one factor as desired, the first component of the estimated factor loading vector is strictly larger than the other components, as formally described in the following proposition.
\begin{proposition} \label{prop:TM-fail}
Suppose $\bfR_*$ and $\bfF_*$ are given as in Proposition \ref{prop:MRH-fail}, and let the estimate resulting from ($\ref{eq:modified_TM}$) be $\hat{\bfSigma} = \bfR_* + \hat{\bfF}$. Then for all $\lambda>0$ we have
\begin{enumerate}
\item $\rank(\hat{\bfF})=1$ if and only if $\lambda < MN/2$.
\item In that case, if we rewrite $\hat{\bfF}$ as $\bff \bff^{\rm T}$, where $\bff=[f_1 \quad \ldots \quad f_M]^{\rm T}$, then $\forall i >1$, $f_1/f_i$ is greater than 1 and monotonically increasing with $r$. Furthermore, if $\lambda > (M-1)N/2$, then $f_1/f_i = \Omega(r)$.
\end{enumerate}
\end{proposition}
Again, this represents a bias that overemphasizes the variable with large residual variance, even when we incorporate additional information into the formulation. On the contrary, it is easy to see that STM can accurately recover all major factors if similar favorable constraints are incorporated into its formulation (ie., if we set $\bfSigma_{\rm SAM}=\bfSigma_*$ and $\bfT \propto \bfR_*^{-\frac{1}{2}}$ in ($\ref{eq:reg_STM}$) ).

\section{Conclusion}

We proposed factor model estimates UTM and STM, both of which are regularized versions of those that would be produced via PCA.  UTM deals with contexts where residual variances are assumed to be uniform, whereas STM handles variation among residual variances.
Our algorithm for computing the UTM estimate is as efficient as conventional PCA.  For STM, we provide an iterative algorithm with guaranteed convergence.
Computational experiments involving both synthetic and real data demonstrate that the estimates produced by our approach are significantly more accurate than those produced by pre-existing methods.  Further, we provide a theoretical analysis that elucidates the way in which UTM and STM corrects biases induced by alternative approaches.

Let us close by mentioning a few possible directions for further research. 
Our analysis has relied on data being generated by a Gaussian distribution.  It would be useful to understand how things change
if this assumption is relaxed.  Further, in practice estimates are often used to guide subsequent decisions.  It would be interesting
to study the impact of STM on decision quality and whether there are other approaches that fare better in this dimension.  Our recent paper on directed principle component analysis \citep{Kao2012} relates to this.  In some applications, PCA is
used to identify a subspace for dimension reduction.  It would be interesting to understand if and when the subspace identified by STM is more suitable.
Finally, there is a growing body of research on robust variations of factor analysis and PCA. These include the pursuit of sparse factor loadings \citep{Jolliffe03, Zou04, DAspremont04, Johnstone07, Amini08}, and the methods that are resistant to corrupted data \citep{Pison03, Candes09, Xu10}.  It would be interesting to explore connections to this body of work.

\appendix

\section{Proofs}

We first prove a main lemma that will be used in the proof of Theorem \ref{thm:UTM} and Proposition \ref{prop:TM-fail}.

\begin{lemma} \label{lem:GV}
Fixing $\bfV \in \mathbb{D}_{++}^{M}$, consider the optimization problem
\begin{eqnarray}
\max_{\bfG \in \spsdM }  &&  \log p({\cal X} | \bfSigma )  - \lambda {\rm tr}( {\bfG} ) \label{eq:reg_cvx_penalty_V} \\
{\rm s.t.} && \bfSigma^{-1} = \bfV - \bfG. \nonumber
\end{eqnarray}
Let $\bfG_\bfV$ be the solution to (\ref{eq:reg_cvx_penalty_V}), $\lambda'=2\lambda/N$, 
and $\bfU \bfD \bfU^{\rm T}$ be an eigendecomposition of matrix $ \bfV^{\frac{1}{2}} \left( \bfSigma_{\rm SAM} - \lambda' \bfI \right) \bfV^{\frac{1}{2}}$
with $\bfU$ orthonormal.
Then we have $ (\bfV - \bfG_\bfV)^{-1} = \bfV^{-\frac{1}{2}} \bfU \bfL \bfU^{\rm T} \bfV^{-\frac{1}{2}} $,
where $\bfL$ is a diagonal matrix with entries $ \bfL_{i,i} = \max \left\{ \bfD_{i,i}, 1 \right\}, \forall i=1,\ldots, M$.

\end{lemma}
\begin{proof}
We can rewrite (\ref{eq:reg_cvx_penalty_V}) as
\begin{eqnarray*} 
\min_{ \bfG } && -\log \det \left( \bfV - \bfG \right) + \mbox{tr}( \left( \bfV - \bfG \right) \bfSigma_{\rm SAM})+ \lambda' \mbox{tr}(\bfG) \\
\mbox{s.t.} && \bfG \in \mathbb{S}_+^M
\end{eqnarray*}
Now associate a Lagrange multiplier $\bfOmega \in \spsdM $ with the $\bfG\in \spsdM$ constraint and write down the Lagrangian as
$$ {\cal L}(\bfG, \bfOmega) = -\log \det \left( \bfV - \bfG \right) + \mbox{tr}( \left( \bfV - \bfG \right) \bfSigma_{\rm SAM})+ \lambda' \mbox{tr}(\bfG) - \mbox{tr}(\bfOmega\bfG). $$
Let $\bfOmega^*$ denote the dual solution. By KKT conditions we have: 
\begin{eqnarray}
\nabla_\bfG {\cal L} \Big|_{\bfG_\bfV, \bfOmega^*} & = & ( \bfV -\bfG_\bfV)^{-1} - \bfSigma_{\rm SAM} + \lambda' \bfI - \bfOmega^* = 0 \label{eq:gen_KKT1} \\
\bfOmega^*, \bfG_\bfV & \in & \spsdM \label{eq:gen_KKT3} \\
\mbox{tr}(\bfOmega^*\bfG_\bfV) &=& 0. \label{eq:gen_KKT4}
\end{eqnarray}
Recall that
\begin{equation}
(\bfV - \bfG_\bfV)^{-1} = (\bfV^{\frac{1}{2}} \bfV^{\frac{1}{2}} - \bfG_\bfV)^{-1} = \bfV^{-\frac{1}{2}} ( \bfI - \bfV^{-\frac{1}{2}} \bfG_\bfV \bfV^{-\frac{1}{2}} )^{-1} \bfV^{-\frac{1}{2}}. \label{eq:inv_VG}
\end{equation}
Plugging this into (\ref{eq:gen_KKT1}) we get
$$ \bfV^{-\frac{1}{2}} (\bfI - \bfV^{-\frac{1}{2}} \bfG_\bfV \bfV^{-\frac{1}{2}})^{-1} \bfV^{-\frac{1}{2}}  = \bfSigma_{\rm SAM} - \lambda' \bfI + \bfOmega^* $$
and so
\begin{eqnarray} \label{eq:I_LGL}
\left(\bfI - \bfV^{-\frac{1}{2}} \bfG_\bfV \bfV^{-\frac{1}{2}} \right)^{-1} & = & \bfV^{\frac{1}{2}} \left( \bfSigma_{\rm SAM} - \lambda' \bfI + \bfOmega^* \right) \bfV^{\frac{1}{2}} \\
& = & \bfV^{\frac{1}{2}} \left( \bfSigma_{\rm SAM} - \lambda' \bfI \right) \bfV^{\frac{1}{2}} + \bfV^{\frac{1}{2}} \bfOmega^* \bfV^{\frac{1}{2}}. \nonumber
\end{eqnarray}
By (\ref{eq:gen_KKT3}), both $ \bfV^{-\frac{1}{2}} \bfG_\bfV \bfV^{-\frac{1}{2}}$ and $\bfV^{\frac{1}{2}} \bfOmega^* \bfV^{\frac{1}{2}}$ are in $\spsdM$.
Let an eigendecomposition of \\
$\bfV^{-\frac{1}{2}} \bfG_\bfV \bfV^{-\frac{1}{2}} $ be $\bfA \bfQ \bfA^{\rm T}$ for which $\bfA=[ \bfa_1 \quad \ldots \quad \bfa_M ]$ is orthonormal and $\bfQ_{i,i}\geq 0, i=1,\ldots, M$. Using (\ref{eq:gen_KKT4}) and the fact that trace is invariant under cyclic permutations, we have
\begin{eqnarray*}
0 & = & \mbox{tr}(\bfOmega^*\bfG_\bfV) = \mbox{tr}(\bfOmega^* \bfV^{\frac{1}{2}} \bfV^{-\frac{1}{2}} \bfG_\bfV \bfV^{-\frac{1}{2}} \bfV^{\frac{1}{2}} ) \\
& = & \mbox{tr}\left(  ( \bfV^{\frac{1}{2}} \bfOmega^* \bfV^{\frac{1}{2}} ) ( \bfV^{-\frac{1}{2}} \bfG_\bfV \bfV^{-\frac{1}{2}} ) \right) = \mbox{tr}\left(  ( \bfV^{\frac{1}{2}} \bfOmega^* \bfV^{\frac{1}{2}} ) \bfA \bfQ \bfA^{\rm T} \right) \\
& = & \mbox{tr}\left(  \bfA^{\rm T}( \bfV^{\frac{1}{2}} \bfOmega^* \bfV^{\frac{1}{2}} ) \bfA \bfQ \right)
= \sum_{i=1}^M \bfQ_{i,i} \bfa_i^{\rm T} ( \bfV^{\frac{1}{2}} \bfOmega^* \bfV^{\frac{1}{2}} ) \bfa_i.
\end{eqnarray*}
Since $\bfQ_{i,i} \geq 0$ and $\bfV^{\frac{1}{2}} \bfOmega^* \bfV^{\frac{1}{2}} \in \spsdM$, we can deduce
$$ \bfa_i^{\rm T} (\bfV^{\frac{1}{2}} \bfOmega^* \bfV^{\frac{1}{2}} ) \bfa_i = 0 \mbox{ if } \bfQ_{i,i}>0, \quad \forall i=1,\ldots, M .$$
Let ${\cal I}_+ = \{ i: \bfQ_{i,i} > 0 \}$. Because $\bfV^{\frac{1}{2}}\bfOmega^* \bfV^{\frac{1}{2}} $ is positive semidefinite, for all $i_0\in {\cal I}_+$ we also have $ \bfV^{\frac{1}{2}} \bfOmega^* \bfV^{\frac{1}{2}} \bfa_{i_0}=0$. Furthermore, since
$$ (\bfI- \bfV^{-\frac{1}{2}} \bfG_\bfV \bfV^{-\frac{1}{2}} )^{-1} = \bfA \mbox{diag}\left(\frac{1}{1-\bfQ_{1,1}},\ldots, \frac{1}{1-\bfQ_{M,M}} \right) \bfA^{\rm T}$$
multiplying both sides of (\ref{eq:I_LGL}) by $\bfa_{i_0}$ leads to
$$ \frac{\bfa_{i_0} }{1-\bfQ_{i_0,i_0}}  = \bfV^{\frac{1}{2}} \left( \bfSigma_{\rm SAM} - \lambda' \bfI \right) \bfV^{\frac{1}{2}} \bfa_{i_0}  $$
which shows $\bfa_{i_0}$ is an eigenvector of $\bfV^{\frac{1}{2}} \left( \bfSigma_{\rm SAM} -\lambda' \bfI \right) \bfV^{\frac{1}{2}} $.  Recall that $\bfV^{\frac{1}{2}} \left( \bfSigma_{\rm SAM} -\lambda' \bfI \right) \bfV^{\frac{1}{2}} = \bfU \bfD \bfU^\tp$. Without loss of generality, we can take $\bfa_i = \bfu_i $ for all $ i \in \calI_+$.  Indeed, since $ {\cal I}_+ = \{ i: \bfQ_{i,i} > 0 \}$, we have
$$ \bfA \bfQ \bfA^\tp = \sum_{i=1}^M \bfQ_{i,i}\bfa_i \bfa_i^\tp  = \sum_{i\in \calI_+} \bfQ_{i,i} \bfa_i \bfa_i^\tp + \sum_{j \notin \calI_+} 0 \cdot \bfa_j \bfa_j^\tp = \sum_{i\in \calI_+} \bfQ_{i,i} \bfu_i \bfu_i^\tp + \sum_{j \notin \calI_+} 0 \cdot \bfu_j \bfu_j^\tp $$
which implies we can further take $\bfA=\bfU$. This gives
\begin{equation}
(\bfI- \bfV^{-\frac{1}{2}} \bfG_\bfV \bfV^{-\frac{1}{2}} )^{-1}= \bfU \bfL \bfU^{\rm T}, \label{eq:ULU}
\end{equation}
where $\bfL$ is a diagonal matrix with entries $\bfL_{i,i} = \frac{1}{1-\bfQ_{i,i}} $, for $i=1,\ldots, M$. Plugging this into (\ref{eq:I_LGL}) we have
\begin{equation} \label{eq:BTB}
\bfU \bfL \bfU^{\rm T} = \bfU \bfD \bfU^{\rm T} + \bfV^{\frac{1}{2}} \Omega^* \bfV^{\frac{1}{2}}.
\end{equation}
For any $i_0 \in {\cal I}_+$, multiplying the both sides of the above equation by $\bfu_{i_0}$ results in
$$ \bfL_{i_0,i_0} \bfu_{i_0}= \bfD_{i_0,i_0} \bfu_{i_0} + 0 $$
which implies $ \bfL_{i_0,i_0} = \bfD_{i_0,i_0} $, or more generally
$$ \bfL_{i,i} = \left\{ \begin{array}{rl}
\bfD_{i,i} & \mbox{if } \bfQ_{i,i} >0 \\
1 & \mbox{otherwise} \\
\end{array} \right.
,\quad i=1,\ldots, M.$$
Since $\bfL_{i,i} = \frac{1}{1 - \bfQ_{i,i}} \geq 1$, to see  $\bfL_{i,i} = \max \left\{ \bfD_{i,i}, 1 \right\}$, it remains to show $\bfL_{i,i} \geq \bfD_{i,i}$ for all $i$.
This follows by rearranging (\ref{eq:BTB})
\begin{eqnarray*}
&& \bfU \bfL \bfU^{\rm T} - \bfU \bfD \bfU^{\rm T} = \bfV^{\frac{1}{2}} \bfOmega^* \bfV^{\frac{1}{2}} \succeq 0 \\
& \Rightarrow & \bfL - \bfD \succeq 0 \Rightarrow \bfL_{i,i} \geq \bfD_{i,i}, \quad \forall i=1,\ldots, M.
\end{eqnarray*}
Finally, plugging (\ref{eq:ULU}) into (\ref{eq:inv_VG}) completes the proof.
\end{proof}

\begin{theorem_dup}
$\bfSigma_{\rm SAM}$ and $\bfSigma_{\rm UTM}^\lambda$ share the same
trace and eigenvectors, and letting the eigenvalues of the two matrices, sorted in decreasing order, be denoted by $s_1,\ldots,s_M$ and $h_1,\ldots,h_M$, respectively, we have
$$h_m = \max \left\{ s_m - \frac{2\lambda}{N},\frac{1}{\hat{v}} \right\}, \mbox{ for } m=1,\ldots, M. $$
\end{theorem_dup}
\begin{proof}
Let $\bfU \diag(s_1, \ldots, s_M) \bfU^{\rm T}$ be an eigendecomposition of $\bfSigma_{\rm SAM}$ such that $\bfU$ is orthonormal.
Define $\bfV=\hat{v}\bfI$, and note that an eigendecomposition of matrix
$$ \bfV^{\frac{1}{2}} \left( \bfSigma_{\rm SAM} - \frac{2\lambda}{N} \bfI \right) \bfV^{\frac{1}{2}} = \hat{v} \left( \bfSigma_{\rm SAM} - \frac{2\lambda}{N} \bfI \right)$$
can be written as $\bfU\bfD\bfU^{\rm T}$,  where $ \bfD_{i,i} = \hat{v} (s_i -  2\lambda/N), i=1,\ldots, M$. By Lemma \ref{lem:GV} we have
$$ \bfSigma_{\rm UTM}^\lambda = (\hat{v}\bfI - \hat{\bfG})^{-1} = \bfV^{-\frac{1}{2}} \bfU \bfL \bfU^{\rm T} \bfV^{-\frac{1}{2}} = \frac{1}{\hat{v}} \bfU \bfL \bfU^{\rm T}$$
where $\bfL \in \mathbb{D}^M_+$, and $\bfL_{i,i} = \max\{ \bfD_{i,i} , 1\} = \max\{ \hat{v}(s_i-2\lambda/N) ,1 \} = \hat{v} \max\{ s_i - 2\lambda/N , 1/\hat{v} \} = \hat{v} h_i, \forall i=1,\ldots, M$. Therefore, 
$$ \bfSigma_{\rm UTM}^\lambda = \frac{1}{\hat{v}} \bfU \bfL \bfU^{\rm T} = \bfU \bfH \bfU^{\rm T}.$$
where $\bfH =\diag(h_1,\ldots, h_M)$, as desired.

Furthermore, recall that we impose no constraint on $v$ when solving UTM, and as a result the partial derivative of the objective function with respect to $v$ should vanish at $\hat{v}$. That is,
$$ \frac{ \partial }{\partial v} \left( \log p({\cal X} | \bfSigma) - \lambda {\rm tr}( {\bfG} )\right) \Big|_{\hat{\bfG},\hat{v}} = -\frac{N}{2} \left( \mbox{tr}(\bfSigma_{\rm SAM}) - \mbox{tr}\left( ( \hat{v}\bfI-\hat{\bfG})^{-1} \right) \right) = 0 $$
which implies $\mbox{tr}(\bfSigma_{\rm SAM}) = \mbox{tr}\left( ( \hat{v}\bfI-\hat{\bfG})^{-1} \right) = \mbox{tr}\left( \bfSigma_{\rm UTM}^\lambda \right)$.
\end{proof}

\begin{theorem_dup}
For all $K$, scalars $\ell_1 > \ell_2 > \cdots > \ell_K > \sigma^2 > 0$, $\rho \in (0,1)$, sequences $N_{(M)}$ such that $|M/N_{(M)} - \rho| = o(1/\sqrt{N_{(M)}})$, covariance matrices $\bfSigma_*^{(M)} \in \spsdM$ with eigenvalues $\ell_1,\ldots,\ell_K,\sigma^2,\ldots, \sigma^2$, and $i$ such that
$\ell_i > (1+\sqrt{\rho}) \sigma^2$, there exists $\epsilon_i \in (0,2\sigma^2/(\ell_i-\sigma^2))$ such that  
$h_i^* \overset{p}{\longrightarrow} s_i - (2+\epsilon_i) \rho \sigma^2$.
\end{theorem_dup}
\begin{proof}
Let $\bfA \bfL \bfA^{\rm T}$ be an eigendecomposition of $\bfSigma_*$, where $\bfA = [\bfa_1 \quad \ldots \quad \bfa_M ]$ is orthonormal and $\bfL = \diag(\ell_1, \ldots, \ell_K, \sigma^2, \ldots, \sigma^2)$.
Recall that
\begin{eqnarray}
h_i^* & = & \bfb_i^{\rm T} \bfSigma_* \bfb_i = \bfb_i^{\rm T} \left( \sum_{j=1}^K \bfa_j \ell_j \bfa_j^{\rm T} + \sum_{j=K+1}^M \bfa_j \sigma^2 \bfa_j^{\rm T} \right) \bfb_i \nonumber \\
& = & \sum_{j=1}^K \ell_j ( \bfb_i^{\rm T} \bfa_j )^2 + \sigma^2 \sum_{j=K+1}^M ( \bfb_i^{\rm T} \bfa_j )^2. \label{eq:h_decomposed}
\end{eqnarray}
Using Theorem 4 in \citet{Paul07}, we have
\begin{equation} \label{eq:h_term1}
( \bfb_i^{\rm T} \bfa_i )^2 \overset{\rm a.s.}{\longrightarrow} \left( 1 - \frac{\rho \sigma^4}{ (\ell_i - \sigma^2)^2 } \right) \big/ \left( 1 + \frac{\rho \sigma^2}{ \ell_i - \sigma^2 } \right).
\end{equation}
Furthermore, decomposing $\bfb_i$ into $\tilde{\bfb}_i + \tilde{\bfb}_i^\perp$, where $\tilde{\bfb}_i \in \mbox{span}(\bfa_1,\ldots, \bfa_K)$ and $\tilde{\bfb}_i^\perp \in \mbox{span}(\bfa_{K+1},$ $\ldots,$ $\bfa_M)$, by Theorem 5 in \citet{Paul07} we have
$ \frac{ \tilde{\bfb}_i }{ \| \tilde{\bfb}_i \|}  \overset{\rm p}{\longrightarrow} \bfa_i $. This implies if $1\leq j \leq K, j\neq i$,
\begin{equation} \label{eq:h_term2}
\bfb_i^{\rm T} \bfa_j = \tilde{\bfb}_i^{\rm T} \bfa_j \overset{\rm p}{\longrightarrow} \| \tilde{\bfb}_i \| \bfa_i^{\rm T} \bfa_j = 0
\end{equation}
and for $K<j\leq M$,
\begin{equation} \label{eq:h_term3}
\sum_{j=K+1}^M ( \bfb_i^{\rm T} \bfa_j )^2 = 1 -  \left\| \sum_{j=1}^K \bfb_i^{\rm T} \bfa_j \right\|^2  \overset{\rm p}{\longrightarrow} 1 - (\bfb_i^{\rm T} \bfa_i)^2.
\end{equation}

Plugging (\ref{eq:h_term1}), (\ref{eq:h_term2}), and (\ref{eq:h_term3}) into (\ref{eq:h_decomposed}) we arrive at
\begin{eqnarray} \label{eq:hl}
h_i^* & \overset{\rm p}{\longrightarrow} & \ell_i \left( 1 - \frac{\rho \sigma^4}{ (\ell_i - \sigma^2)^2 } \right) \big/ \left( 1 + \frac{\rho \sigma^2}{ \ell_i - \sigma^2 } \right) + \sigma^2 \left( 1- \left( 1 - \frac{\rho \sigma^4}{ (\ell_i - \sigma^2)^2 } \right) \big/ \left( 1 + \frac{\rho \sigma^2}{ \ell_i - \sigma^2 } \right) \right) \nonumber \\
& = & \frac{\ell_i}{1+ \rho \sigma^2 / (\ell_i-\sigma^2) } \nonumber \\
& = & \ell_i - \left( 1 + \frac{(1-\rho) \sigma^2}{\ell_i- (1-\rho)\sigma^2 } \right) \rho\sigma^2 .
\end{eqnarray}

By Theorem 1.1 in \citet{Baik06} we have
\begin{equation} \label{eq:sl}
s_i \overset{\rm a.s}{\longrightarrow} \ell_i + \frac{\rho \ell_i \sigma^2 }{ \ell_i - \sigma^2} 
 = \ell_i + \left( 1 + \frac{\sigma^2}{\ell_i-\sigma^2} \right) \rho \sigma^2 .
\end{equation}
Finally, combining (\ref{eq:hl}) and (\ref{eq:sl}) yields
$$ h_i^* \overset{\rm p}{\longrightarrow} s_i  - (2 + \epsilon_i) \rho \sigma^2 $$
where $ \epsilon_i =  \frac{(1-\rho) \sigma^2}{\ell_i- (1-\rho)\sigma^2 } + \frac{\sigma^2}{\ell_i-\sigma^2} $. It is easy to see $0< \epsilon_i  < \frac{\sigma^2}{\ell_i-\sigma^2} + \frac{\sigma^2}{\ell_i-\sigma^2} $, as desired.
\end{proof}

\begin{proposition_dup}
For any fixed $N, K$, scalars $\ell_1 \geq \ell_2 \geq \cdots \geq \ell_K \geq \sigma^2 > 0$, and any sequence of covariance matrices $\bfSigma_*^{(M)} \in \spsdM$ with eigenvalues $\ell_1,\ldots,\ell_K,\sigma^2,\ldots, \sigma^2$, we have
$$ \frac{ {\rm tr} \bfSigma_{\rm SAM} }{ {\rm tr} \bfSigma_* } \overset{a.s.}{\longrightarrow} 1$$
as $M \tends \infty$ and 
$$ {\rm Pr} \left( \left| \frac{ {\rm tr} \bfSigma_{\rm SAM} }{ {\rm tr} \bfSigma_* } - 1 \right| \geq \epsilon \right) \leq 2\exp\left( - N \epsilon^2 \Omega(M) \right). $$
\end{proposition_dup}
\begin{proof}
Let $\bfA \bfL \bfA^{\rm T}$ be an eigendecomposition of $\bfSigma_*$, where $\bfA = [\bfa_1 \quad \ldots \quad \bfa_M ]$ is orthonormal and $\bfL = \diag(\ell_1, \ldots, \ell_K, \sigma^2, \ldots, \sigma^2)$. Thus,
$\mbox{tr} \bfSigma_* = \mbox{tr} \bfL = \sum_{k=1}^K \ell_k + (M-K) \sigma^2 $.
Note that
$$\mbox{tr} \bfSigma_{\rm SAM} = \mbox{tr} ( \bfA \bfA^{\rm T} \bfSigma_{\rm SAM} ) = \mbox{tr} ( \bfA^{\rm T} \bfSigma_{\rm SAM} \bfA  ) = \sum_{i=1}^M \bfa_i^{\rm T} \bfSigma_{\rm SAM} \bfa_i = \sum_{i=1}^M \frac{1}{N}  \sum_{n=1}^N \left( \bfa_i^{\rm T} \bfx_{(n)} \right)^2. $$
Since $\bfx_{(n)} \sim {\cal N}(0, \bfSigma_*) $, we can think of each $\bfx_{(n)}$ as generated by
$ \bfx_{(n)} = \bfA \bfz_{(n)}$,
where $\bfz_{(n)}$ is sampled iid from ${\cal N}(0, \bfL)$. This leads to
$$ \mbox{tr} \bfSigma_{\rm SAM} = \sum_{i=1}^M \frac{1}{N} \sum_{n=1}^N \bfz_{(n),i}^2 = \sum_{i=1}^M \frac{ \bfL_{i,i} }{N} w_i^2, $$
where $w_i^2$'s are i.i.d. samples from $\chi^2_N$. Therefore,
$$ \lim_{M\tends \infty} \frac{ {\rm tr} \bfSigma_{\rm SAM} }{ {\rm tr} \bfSigma_* } = \lim_{M\tends \infty} \frac{ \sum_{i=1}^M \frac{ \bfL_{i,i} }{N} w_i^2 }{ \sum_{k=1}^K \ell_k + (M-K) \sigma^2 }
= \lim_{M\tends \infty} \frac{ \sum_{k=1}^K  \ell_k \frac{w_i^2}{N} +  \sum_{i=K+1}^M \sigma^2 \frac{ w_i^2 }{N} }{ \sum_{k=1}^K \ell_k + (M-K) \sigma^2 }. $$
Since the first terms in the denominator and the numerator are bounded and do not scale with $K$, we can drop them in the limit and rewrite
$$ \lim_{M\tends \infty} \frac{ {\rm tr} \bfSigma_{\rm SAM} }{ {\rm tr} \bfSigma_* } = \lim_{M\tends \infty} \frac{ \sum_{i=K+1}^M \sigma^2 \frac{ w_i^2 }{N} }{ (M-K) \sigma^2 } = \lim_{M\tends \infty} \frac{1}{M-K} \sum_{i=K+1}^M
\frac{w_i^2}{N} = 1 \mbox{ (w. p. 1)} $$
due to the strong law of large numbers and the fact that $\E[w_i^2]=N$.

To prove the second part of the proposition, let us rewrite $ w_i^2 = \sum_{n=1}^N \tilde{w}_{i,n}^2 $, where $\tilde{w}_{i,n}$ are i.i.d samples from $\calN(0,1)$. Therefore,
$$ \tr \bfSigma_{\rm SAM} - \tr \bfSigma_* = \sum_{i=1}^M \sum_{n=1}^N \frac{ \bfL_{i,i} }{N} \tilde{w}^2_{i,n} -\sum_{i=1}^M \bfL_{i,i} = \sum_{i=1}^M \sum_{n=1}^N \frac{ \bfL_{i,i} }{N} \left( \tilde{w}^2_{i,n} -1 \right). $$
By the exponential inequality for chi-square distributions \citep{Laurent00}, we have 
$$ {\rm Pr}( \left| \tr \bfSigma_{\rm SAM} - \tr \bfSigma_* \right| \geq 2 \xi \sqrt{\tau} ) \leq 2 \exp(-\tau), \quad \forall \tau > 0, $$
where $\xi = \sqrt{  \sum_{i=1}^M \sum_{n=1}^N \left( \frac{\bfL_{i,i}}{N} \right)^2 } = \sqrt{ \frac{1}{N} \sum_{i=1}^M \bfL_{i,i}^2 } $. Taking $\tau = \left( \frac{ \epsilon \tr \bfSigma_*}{ 2 \xi } \right)^2 $, we can rewrite the above inequality as
$$ {\rm Pr} \left( \left| \frac{ {\rm tr} \bfSigma_{\rm SAM} }{ {\rm tr} \bfSigma_* } - 1 \right| \geq \epsilon \right) \leq 2 \exp\left(  -\left( \frac{ \epsilon \tr \bfSigma_*}{ 2 \xi } \right)^2 \right) = 2 \exp \left( - N \epsilon^2 \frac{ \left( \sum_{i=1}^M \bfL_{i,i} \right)^2  }{ 4 \sum_{i=1}^M \bfL_{i,i}^2  } \right) . $$
The desired result then follows straightforwardly from the fact that
$$ \frac{ \left( \sum_{i=1}^M \bfL_{i,i} \right)^2  }{ \sum_{i=1}^M \bfL_{i,i}^2  } = \frac{ \left( M\sigma^2 + \kappa_1 \right)^2 }{ M\sigma^4 + \kappa_2 } = \Omega(M), $$
since $\kappa_1 = \sum_{i=1}^K (\ell_i-\sigma^2)$ and $\kappa_2 = \sum_{i=1}^K (\ell_i^2-\sigma^4) $ are both constants.
\end{proof}

\begin{proposition_dup}
Suppose $\bfR_* = \diag( r,1,1,\ldots, 1)$ and $\bfF_* = {\bf 1}{\bf 1}^\tp$, $r > 1$. Let $\bff = [f_1 \quad \ldots \quad f_M]^{\rm T}$ be the top eigenvector of $\bfSigma_*$. Then we have $ f_1 / f_i = \Omega(r), \forall i>1$.
\end{proposition_dup}
\begin{proof}
Note that $f$ is the solution to the following optimization problem
\begin{eqnarray*}
\max_{\bfu \in \mathbb{R}^M} && \bfu^{\rm T} (\bfR_*+\bfF_*) \bfu \\
\mbox{s.t.} && \|\bfu\|_2 = 1
\end{eqnarray*}
and the objective function can written be as $ r f_1^2 + \sum_{i=2}^M f_i^2 + \left(\sum_{i=1}^M f_i \right)^2$. By symmetry, we have $f_2=f_3=\ldots=f_M$. To simplify notation, let us represent $\bff$ as $[x \quad y \quad y \quad \ldots \quad y]^{\rm T} $. Suppose the largest eigenvalue is $q$. By definition we have $(\bfR_*+\bfF_*)\bff = q \bff$, or equivalently
\begin{eqnarray*}
(r+1)x + (M-1)y &=& q x \\
x + M y &=& q y.
\end{eqnarray*}
Solving the above equations leads to
$$ q = \frac{1}{2} \left(  M+r+1+ \sqrt{ (M+r+1)^2-4(Mr+1) } \right) = \Omega(r), $$
and plugging this back to the above equations yields
$$ x/y = q-M = \Omega(r).$$
\end{proof}

\begin{proposition_dup}
Suppose $\bfR_*$ and $\bfF_*$ are given as in Proposition \ref{prop:MRH-fail}, and let the estimate resulting from ($\ref{eq:modified_TM}$) be $\hat{\bfSigma} = \bfR_* + \hat{\bfF}$. Then for all $\lambda>0$ we have
\begin{enumerate}
\item $\rank(\hat{\bfF})=1$ if and only if $\lambda < MN/2$.
\item In that case, if we rewrite $\hat{\bfF}$ as $\bff \bff^{\rm T}$, where $\bff=[f_1 \quad \ldots \quad f_M]^{\rm T}$, then $\forall i >1$, $f_1/f_i$ is greater than 1 and monotonically increasing with $r$. Furthermore, if $\lambda > (M-1)N/2$, then $f_1/f_i = \Omega(r)$.
\end{enumerate}
\end{proposition_dup}

\begin{proof}
Let $\lambda'=2\lambda/N$, $\bfC=\bfR_*^{-\frac{1}{2}} \left( \bfSigma_* - \lambda' \bfI \right) \bfR_*^{-\frac{1}{2}}$, $\bfU \bfD \bfU^{\rm T}$ be an eigendecomposition of $\bfC$, and $\bfL$ be the diagonal matrix with entries $ \bfL_{i,i} = \max \left\{ \bfD_{i,i}, 1 \right\}, \forall i=1,\ldots, M.$
Applying Lemma \ref{lem:GV} with $\bfSigma_{\rm SAM} = \bfSigma_*$ and $\bfV = \bfR_*^{-1}$ we have
$$ \hat{\bfSigma} = \bfR_*^{\frac{1}{2}} \bfU \bfL \bfU^{\rm T} \bfR_*^{\frac{1}{2}} = \bfR_*^{\frac{1}{2}} \bfU \bfI \bfU^{\rm T} \bfR_*^{\frac{1}{2}} + \bfR_*^{\frac{1}{2}} \bfU ( \bfL - \bfI ) \bfU^{\rm T} \bfR_*^{\frac{1}{2}} = \bfR_* + \bfR_*^{\frac{1}{2}} \bfU ( \bfL - \bfI ) \bfU^{\rm T} \bfR_*^{\frac{1}{2}} $$
and therefore
$$ \hat{\bfF} = \bfR_*^{\frac{1}{2}} \bfU ( \bfL - \bfI ) \bfU^{\rm T} \bfR_*^{\frac{1}{2}}. $$
Since $\bfL_{i,i} = \max \left\{ \bfD_{i,i}, 1 \right\}$, we further have $\mbox{rank}(\hat{\bfF}) = \mbox{rank}(\bfL - \bfI ) = | \left\{i: \bfD_{i,i}>1 \right\} | $.
Recall that $\bfD$ denotes the eigenvalues of matrix $\bfC$, which can be written as 
\begin{eqnarray*}
\bfR_*^{-\frac{1}{2}} \left( \bfSigma_* - \lambda' \bfI \right) \bfR_*^{-\frac{1}{2}} & = & \bfR_*^{-\frac{1}{2}} \bfSigma_* \bfR_*^{-\frac{1}{2}} - \lambda' \bfR_*^{-1} \\
& = & \bfR_*^{-\frac{1}{2}}( \bfR_* + \bfF_* ) \bfR_*^{-\frac{1}{2}} - \lambda' \left( \bfI - \left(1-\frac{1}{r} \right) \bfe_1\bfe_1^{\rm T} \right) \\
& = & \bfI + \bfa \bfa^{\rm T} - \lambda' \bfI + \lambda' \left( 1-\frac{1}{r} \right)\bfe_1 \bfe_1^{\rm T} \\
& = & (1-\lambda') \bfI + \bfa \bfa^{\rm T} + \lambda' \left( 1-\frac{1}{r} \right)\bfe_1 \bfe_1^{\rm T}
\end{eqnarray*}
where $\bfa = \left[ \frac{1}{\sqrt{r}} \quad 1 \quad \ldots \quad 1 \right]^{\rm T}$ and $\bfe_1= [ 1 \quad 0 \quad \ldots \quad 0 ]^{\rm T}$. Let $\bfA = \bfa \bfa^{\rm T} + \lambda' \left( 1-\frac{1}{r} \right)\bfe_1 \bfe_1^{\rm T}$. Since $\bfC = (1-\lambda')\bfI + \bfA$, we know $\bfC$ and $\bfA$ share the same eigenvectors, and the corresponding eigenvalues differ by $(1-\lambda')$. Thus, the number of the eigenvalues of $\bfC$ that are greater than 1 is equal to the number of the eigenvalues of $\bfA$ that are greater than $\lambda'$. However, $\mbox{rank}(\bfA)= \mbox{rank}( \bfa \bfa^{\rm T} + \lambda' \left( 1-\frac{1}{r} \right)\bfe_1 \bfe_1^{\rm T} ) = 2$, which implies $\bfA$ has only 2 non-zero eigenvalues. Let $q$ be one of them. By symmetry, we can denote the corresponding eigenvector by $\bfu= [ x \quad y \quad \ldots \quad y ]^{\rm T}$.
Then we have $ \bfA \bfu = q \bfu $, which leads to
\begin{eqnarray}
\frac{x}{r} + \frac{ (M-1) y }{\sqrt{r}} + \left( \lambda'- \frac{\lambda'}{r} \right)x & = & qx \nonumber \\
\frac{x}{\sqrt{r}} + (M-1) y & = & qy. \label{eq:qxy}
\end{eqnarray}
After eliminating $x$ and $y$ we arrive at the following equation
$$ rq^2 - ( (M-1)r + (r-1)\lambda' + 1 )q + \lambda'(M-1)(r-1) = 0. $$
Let $s(q)$ denote the left-hand-side of the above equation. It is easy to see that its discriminant $\Delta > 0$, and therefore the equation $s(q)=0$ has two distinct real roots, each corresponding to one of the non-zero eigenvalues of $\bfA$. Recall that $\rank(\hat{\bfF})$ equals the number of the eigenvalues of $\bfA$ that are greater than $\lambda'$, which is equal to the number of the roots of $s(q)=0$ that are greater than $\lambda'$. Thus, $\rank(\hat{\bfF}) = 1 $ if and only if one root of $s(q)=0$ is greater than $\lambda'$ and the other is less than $\lambda'$. This is equivalent to
$$ s(\lambda') < 0 \iff \lambda'(\lambda'-M) < 0 \iff \lambda' < M \iff \lambda < \frac{MN}{2}. $$

To prove the second part of this proposition, let $q_+$ be the greatest eigenvalue of $\bfA$, or equivalently the greater root of $s(q)=0$, and $\bfu_+$ be the corresponding eigenvector. That is,
$$ q_+ = \frac{1}{2r} \left( (M-1)r + (r-1)\lambda' + 1 + \sqrt{\Delta} \right) $$
and $ \bfu_+ = [ x_+ \quad y_+ \quad \ldots \quad y_+ ]^{\rm T}$, where $(x_+,y_+)$ is a solution of $(x,y)$ in ($\ref{eq:qxy}$) given $q=q_+$.
Recall that
$$ \hat{\bfF} = \bfR_*^{\frac{1}{2}} \bfU ( \bfL - \bfI ) \bfU^{\rm T} \bfR_*^{\frac{1}{2}} = \bfR_*^{\frac{1}{2}} \left( ( q_+ - \lambda' ) \bfu_+ \bfu_+^{\rm T} \right) \bfR_*^{\frac{1}{2}} , $$
which leads to $\bff= (q_+ - \lambda' )^{\frac{1}{2}} \bfR_*^{\frac{1}{2}} \bfu_+ = (q_+ - \lambda')^{\frac{1}{2}} [ \sqrt{r}x_+ \quad y_+ \quad \ldots \quad y_+ ]$, and
$$ \frac{f_1}{f_i} = \frac{ \sqrt{r} x_+ }{ y_+} = r(q_++1-M), \quad \forall i>1. $$
It is easy to show $ \lim_{ \lambda' \tends 0^+ } q_+ = \frac{1}{r} + M-1 $ and as a result $ \lim_{ \lambda' \tends 0^+ } \frac{f_1}{f_i} = 1$. Therefore, to show $\frac{f_1}{f_i}$ is greater than 1 and monotonically increasing with $r$, it is sufficient to show that
$$ \frac{ d \frac{f_1}{f_i} }{d r } > 0, \quad \forall r > 1, \lambda' \in (0, M) . $$
By straight forward algebra we have
$$ \frac{ d \frac{f_1}{f_i} }{d r } = (\lambda' - M + 1 ) + \frac{1}{\sqrt \Delta} \left( ( \lambda'r-\lambda'- Mr + r )(\lambda'-M+1) + (M-1+\lambda')  \right). $$
Now consider three cases:
\begin{packed_enum}
\item $ 0 < \lambda < \frac{(M-1)N}{2} $ : In this case $ 0 < \lambda' < M-1$ and
\begin{eqnarray*}
\frac{ d \frac{f_1}{f_i} }{d r } > 0 & \iff &
\left( ( \lambda'r-\lambda'- Mr + r )(\lambda'-M+1) + (M-1+\lambda')  \right) > (M-1-\lambda'){\sqrt \Delta} \\
& \iff & \left( -\lambda'r + \lambda' + Mr - r  + \frac{M-1+\lambda'}{M-1-\lambda'}  \right)^2 > \Delta .
\end{eqnarray*}
Expanding the both sides of the last inequality yields the desired result.
\item $\lambda = \frac{(M-1)N}{2}$ : In this case $\lambda' = M-1$ and $\frac{ d \frac{f_1}{f_i} }{d r } = \frac{2(M-1) }{\sqrt \Delta } > 0$, as desired.
\item $ \frac{(M-1)N}{2} < \lambda < \frac{MN}{2} $ : In this case $\lambda' - M + 1 \in (0,1)$, and we have 
\begin{eqnarray*}
&& ( \lambda'r-\lambda'- Mr + r )(\lambda'-M+1) + (M-1+\lambda') \\
& = & r(\lambda'-M+1)^2 - \lambda'(\lambda'-M+1) + (\lambda'+M-1) \\
& > & - \lambda' + (\lambda'+M-1) > 0
\end{eqnarray*}
which implies $\frac{ d \frac{f_1}{f_i} }{d r } > (\lambda' - M + 1 )$. Since the derivative is bounded below by a positive constant, we conclude $ \frac{f_1}{f_i} = \Omega(r)$.
\end{packed_enum}
\end{proof}

\section{Experiment Details}

\subsection{Coordinate Ascent Algorithm for STM}

Algorithm \ref{alg:rpca} describes our coordinate ascent method for solving STM.
\begin{algorithm}[h]
\caption{Procedure for solving STM \\ {\bf Input: }${\cal X}, \lambda$ \\ {\bfseries Output: } $\bfSigma_{\rm STM}^{\lambda}$ }
\label{alg:rpca}
\begin{algorithmic}
\STATE $\bfT \leftarrow \bfI $.
\REPEAT
	\STATE $ \bfSigma \leftarrow {\rm UTM} ( \bfT {\cal X}, \lambda ) $
	\STATE $\bfT \leftarrow \argmax_{\bar{\bfT} \in \mathbb{D}_+^M } \log p( \bar{\bfT} {\cal X}|\bfSigma), \mbox{ s.t. } \log\det \bar{\bfT} \geq 0 $
\UNTIL{converge}
\STATE $\bfSigma_{\rm STM}^{\lambda} \leftarrow \bfT^{-1} \bfSigma \bfT^{-\rm T} $
\end{algorithmic}
\end{algorithm}

\subsection{Cross Validation for Synthetic Data Experiment}

For the synthetic data experiment, we select the regularization parameter via the following cross-validation procedure. Let $\theta$ be the regularization parameter to be determined and ${\cal U}({\cal X}, \theta)$ be the learning algorithm that takes as input $({\cal X},\theta)$ and returns a covariance matrix estimate.
We randomly split $\cal X$ into a partial training set ${\cal X}_{\rm T}$ and a validation set ${\cal X}_{\rm V}$, whose sizes are roughly $70\%$ and $30\%$ of $\cal X$, respectively.
For each candidate value of $\theta$, $\hat{\bfSigma}^{\theta}_{\rm T} = {\cal U}({\cal X_{\rm T}},\theta )$ is computed and the likelihood $p({\cal X}_{\rm V}|\hat{\bfSigma}^\theta_{\rm T})$ of the validation set
${\cal X}_{\rm V}$ conditioned on the solution $\hat{\bfSigma}^\theta_{\rm T}$ is evaluated. The value of $\theta$ that maximizes this likelihood is then selected and fed into ${\cal U}({\cal X}, \theta)$ along with the full training set $\cal X$, resulting in our estimate $\hat{\bfSigma}^\theta$.

In our synthetic data experiment, the $K$ for URM/EM/MRH are selected from $\left\{0, 1, \ldots,\right.$ $\left. 15 \right\}$, and the $\lambda$ for UTM/TM/STM are selected from $\{100,120,\ldots,400\}$. These ranges are chosen so that the selected values rarely fall on the extremes.

\subsection{Termination Criteria for Iterative Algorithms}

In our implementation, the EM algorithm terminates when
$ \max_i \left| \frac{ \bfR_{i,i}^{\rm New}-\bfR_{i,i}^{\rm Current} }{ \bfR_{i,i}^{\rm Current} } \right| < 0.001. $
Similarly, the STM algorithm terminates when
$ \max_i \left| \frac{ \bfT_{i,i}^{\rm New}-\bfT_{i,i}^{\rm Current} }{ \bfT_{i,i}^{\rm Current} } \right| < 0.001. $

\subsection{Equivalent Data Requirement}

Consider two learning algorithms $\calU_1$ and $\calU_2$, and $N$ data samples $\calX = \{ \bfx_{(1)}, \ldots, \bfx_{(N)} \}$. Denote the out-of-sample log-likelihood delivered by $\calU$ with $\calX$ by $L(\calU, \calX)$. Suppose $\calU_2$ generally has better performance than $\calU_1$. Algorithm $\ref{alg:edr}$ evaluates the equivalent data requirement of $\calU_2$ with respect to $\calU_1$. In our implementation, we set step size $\alpha=2\%$ for uniform-residual experiment and $\alpha=10\%$ for nonuniform-residual ones.

\begin{algorithm}[h]
\caption{Procedure for evaluating equivalent data requirement \\ {\bf Input: }$\calX, \calU_1, \calU_2$ \\ {\bfseries Output: } $\gamma$ } \label{alg:edr}
\begin{algorithmic}
\STATE $i \leftarrow 0 $
\WHILE {1}
	\STATE $\gamma \leftarrow 1 - i\alpha$
	\STATE $\calX_{i} \leftarrow \{ \bfx_{(1)}, \ldots, \bfx_{(\gamma N)} \} $
	\IF { $L(\calU_2, \calX_i) < L(\calU_1, \calX ) $ }
		\IF { $i > 0$ }
			\STATE $\gamma \leftarrow \gamma + \frac{ L(\calU_1,\calX) - L(\calU_2,\calX_i) }{ L(\calU_2,\calX_{i-1}) - L(\calU_2,\calX_i) }  \alpha $ \quad (interpolation)
		\ENDIF
		\RETURN $\gamma$
	\ENDIF
	\STATE $ i \leftarrow i + 1 $
\ENDWHILE
\end{algorithmic}
\end{algorithm}

\subsection{S\&P500 Data Preprocessing}

Define November 2, 2001 as trading day 1 and August 9, 2007 as trading day 1451. After deleting 47 constituent stocks that are not fully defined over this period, we compute for each stock the normalized log daily returns as follows:
\begin{packed_enum}
	\item Let $y'_{i,j}$ be the adjusted close price of stock $i$ on day $j$,  $i=1,\ldots, 453$ and $j=1,\ldots, 1451$ .
	\item Compute the raw log-daily-return of stock $i$ on day $j$ by
	$$ y''_{i,j} = \log \frac{y'_{i,j+1}}{y'_{i,j}}, \quad i=1, \ldots, 453, \quad j = 1,\ldots, 1450. $$
	\item Let $\bar{y}$ be the smallest number such that at least 99.5\% of all $y''_{i,j}$ are less than or equal to $\bar{y}$. Let $\underline{y}$ be the largest number such that at least 99.5\% of all $y''_{i,j}$'s are greater than or equal to $\underline{y}$. Clip all $y''_{i,j}$ by the interval $[\underline{y}, \bar{y}]$.
	\item Let the volatility of stock $i$ on day $j>50$ be the 10-week rms $ \hat{\sigma}_{i,j} = \sqrt{ \frac{1}{50} \sum\limits_{t=1}^{50} y_{i,j-t}''^2 }$.
	\item Set $ \bfy_{(n)} = \left[ \frac{y''_{1,n+50}}{\hat{\sigma}_{1,n+50}} \quad \ldots \quad \frac{y''_{453,n+50}}{\hat{\sigma}_{453,n+50}} \right]^{\rm T}$ for $n=1, \ldots, 1400$.
\end{packed_enum}

\subsection{Candidates for Regularization Parameter in Real Data Experiment}

In our real data experiment, the $K$ for EM/MRH are selected from $\left\{0, 1, \ldots, 40 \right\}$, and the $\lambda$ for TM/STM are selected from$\{200,210,\ldots, 600\}$. These ranges are chosen so that the selected values never fall on the extremes.

\bibliography{RPCA-ref-02-23-2013}

\begin{thebibliography}{}

\bibitem[\protect\astroncite{Akaike}{1987}]{Akaike87}
Akaike, H. (1987).
\newblock Factor analysis and {AIC}.
\newblock {\em Psychometrika}, 52(3).

\bibitem[\protect\astroncite{Amini and Wainwright}{2009}]{Amini08}
Amini, A.~A. and Wainwright, M.~J. (2009).
\newblock High-dimensional analysis of semidefinite relaxations for sparse
  principal components.
\newblock {\em Ann. Statist.}, 37:2877--2921.

\bibitem[\protect\astroncite{Baik and Silverstein}{2006}]{Baik06}
Baik, J. and Silverstein, J.~W. (2006).
\newblock Eigenvalues of large sample covariance matrices of spiked population
  models.
\newblock {\em Journal of Multivariate Analysis}, 97:2006.

\bibitem[\protect\astroncite{Banerjee et~al.}{2008}]{Banerjee08}
Banerjee, O., Ghaoui, L.~E., and d'Aspremont, A. (2008).
\newblock Model selection through sparse maximum likelihood estimation for
  multivariate gaussian or binary data.
\newblock {\em Journal of Machine Learning Research}, 9:485--516.

\bibitem[\protect\astroncite{Bishop}{1998}]{Bishop98}
Bishop, C.~M. (1998).
\newblock Bayesian {PCA}.
\newblock In Kearns, M.~J., Solla, S.~A., and Cohn, D.~A., editors, {\em
  Advances in Neural Information Processing Systems 11}, pages 382--388. MIT
  Press.

\bibitem[\protect\astroncite{Boyd et~al.}{2011}]{Boyd11}
Boyd, S., Parikh, N., Chu, E., Peleato, B., and Eckstein, J. (2011).
\newblock Distributed optimization and statistical learning via the alternating
  direction method of multipliers.
\newblock {\em Foundations and Trends in Machine Learning}, 3(1):1--122.

\bibitem[\protect\astroncite{Cand{\`e}s et~al.}{2009}]{Candes09}
Cand{\`e}s, E.~J., Li, X., Ma, Y., and Wright, J. (2009).
\newblock Robust principal component analysis?
\newblock {\em Journal of the ACM}, 58(1):1--37.

\bibitem[\protect\astroncite{Chandrasekaran et~al.}{2012}]{Chandrasekaran10}
Chandrasekaran, V., Parrilo, P.~A., and Willsky, A.~S. (2012).
\newblock Latent variable graphical model selection via convex optimization.
\newblock {\em Annals of Statistics}, 40(4):1935--2013.

\bibitem[\protect\astroncite{D'Aspremont et~al.}{2004}]{DAspremont04}
D'Aspremont, A., Ghaoui, L.~E., Jordan, M.~I., and Lanckriet, G. R.~G. (2004).
\newblock A direct formulation for sparse {PCA} using semidefinite programming.
\newblock {\em SIAM Review}, 49(3):434.

\bibitem[\protect\astroncite{Friedman et~al.}{2008}]{Friedman08}
Friedman, J., Hastie, T., and Tibshirani, R. (2008).
\newblock {Sparse inverse covariance estimation with the graphical lasso}.
\newblock {\em Biostatistics}, 9(3):432--441.

\bibitem[\protect\astroncite{Harman}{1976}]{Harman76}
Harman, H.~H. (1976).
\newblock {\em Modern Factor Analysis}.
\newblock University of Chicago Press, third edition.

\bibitem[\protect\astroncite{Hirose et~al.}{2011}]{Hirose11}
Hirose, K., Kawano, S., Konishi, S., and Ichikawa, M. (2011).
\newblock Bayesian information criterion and selection of the number of factors
  in factor analysis models.
\newblock {\em Journal of Data Science}, 9(2).

\bibitem[\protect\astroncite{Johnstone}{2001}]{Johnstone01}
Johnstone, I.~M. (2001).
\newblock On the distribution of the largest eigenvalue in principal components
  analysis.
\newblock {\em Annals of Statistics}, 29:295--327.

\bibitem[\protect\astroncite{Johnstone and Lu}{2007}]{Johnstone07}
Johnstone, I.~M. and Lu, A.~Y. (2007).
\newblock {Sparse Principal Components Analysis}.
\newblock {\em Journal of the American Statistical Association}.

\bibitem[\protect\astroncite{Jolliffe et~al.}{2003}]{Jolliffe03}
Jolliffe, I.~T., Trendafilov, N.~T., and Uddin, M. (2003).
\newblock A modified principal component technique based on the lasso.
\newblock {\em Journal of Computational and Graphical Statistics}, 12:531--547.

\bibitem[\protect\astroncite{Kao and Van~Roy}{2012}]{Kao2012}
Kao, Y.-H. and Van~Roy, B. (2012).
\newblock Directed principle component analysis.

\bibitem[\protect\astroncite{Laurent and Massart}{2000}]{Laurent00}
Laurent, B. and Massart, P. (2000).
\newblock Adaptive estimation of a quadratic functional by model selection.
\newblock {\em Annals of Statistics}, 28(5):1302--1338.

\bibitem[\protect\astroncite{Markowitz}{1952}]{Markowitz52}
Markowitz, H.~M. (1952).
\newblock Portfolio selection.
\newblock {\em Journal of Finance}, 7:77--91.

\bibitem[\protect\astroncite{Minka}{2000}]{Minka00}
Minka, T.~P. (2000).
\newblock Automatic choice of dimensionality for {PCA}.
\newblock In Leen, T.~K., Dietterich, T.~G., and Tresp, V., editors, {\em
  Advances in Neural Information Processing Systems 13}, pages 598--604. MIT
  Press.

\bibitem[\protect\astroncite{Paul}{2007}]{Paul07}
Paul, D. (2007).
\newblock Asymptotics of sample eigenstruture for a large dimensional spiked
  covariance model.
\newblock {\em Statistica Sinica}, 17(4):1617--1642.

\bibitem[\protect\astroncite{Pison et~al.}{2003}]{Pison03}
Pison, G., Rousseeuw, P.~J., Filzmoser, P., and Croux, C. (2003).
\newblock Robust factor analysis.
\newblock {\em Journal of Multivariate Analysis}, 84(1):145 -- 172.

\bibitem[\protect\astroncite{Pourahmadi}{2011}]{Pourahmadi10}
Pourahmadi, M. (2011).
\newblock Covariance estimation: The {GLM} and regularization perspectives.
\newblock {\em Statistical Science}, 26(3):369--387.

\bibitem[\protect\astroncite{Ravikumar et~al.}{2011}]{Ravikumar11}
Ravikumar, P., Raskutti, G., Wainwright, M.~J., and Yu, B. (2011).
\newblock High-dimensional covariance estimation by minimizing {L}1-penalized
  log-determinant.
\newblock {\em Electronic Journal of Statistics}.

\bibitem[\protect\astroncite{Rubin and Thayer}{1982}]{Rubin82}
Rubin, D.~B. and Thayer, D.~T. (1982).
\newblock {EM} algorithm for {ML} factor analysis.
\newblock {\em Psychometrika}, 47(1):69--76.

\bibitem[\protect\astroncite{Tipping and Bishop}{1999}]{Tipping99}
Tipping, M.~E. and Bishop, C.~M. (1999).
\newblock Probabilistic principal component analysis.
\newblock {\em Journal of the Royal Statistical Society, Series B},
  61:611--622.

\bibitem[\protect\astroncite{Xu et~al.}{2010}]{Xu10}
Xu, H., Caramanis, C., and Mannor, S. (2010).
\newblock Principal component analysis with contaminated data: The high
  dimensional case.
\newblock In {\em COLT}, pages 490--502.

\bibitem[\protect\astroncite{Yuan and Lin}{2007}]{Yuan07}
Yuan, M. and Lin, Y. (2007).
\newblock Model selection and estimation in the gaussian graphical model.
\newblock {\em Biometrika}, 94(1):19--35.

\bibitem[\protect\astroncite{Zou et~al.}{2004}]{Zou04}
Zou, H., Hastie, T., and Tibshirani, R. (2004).
\newblock {Sparse Principal Component Analysis}.
\newblock {\em Journal of Computational and Graphical Statistics}, 15.

\end{thebibliography}

\end{document}